\RequirePackage{xcolor}
\documentclass[journal,twoside,web]{IEEEtran}
\usepackage{etoolbox}
\makeatletter
\@ifundefined{color@begingroup}%
  {\let\color@begingroup\relax
   \let\color@endgroup\relax}{}%
\def\fix@ieeecolor@hbox#1{%
  \hbox{\color@begingroup#1\color@endgroup}}
\patchcmd\@makecaption{\hbox}{\fix@ieeecolor@hbox}{}{\FAILED}
\patchcmd\@makecaption{\hbox}{\fix@ieeecolor@hbox}{}{\FAILED}

\definecolor{subsectioncolor}{rgb}{0,0,0}
\usepackage{amsmath,amssymb,amsfonts}
\usepackage{cite}
\usepackage{graphicx}
\usepackage{hyperref}
\hypersetup{hidelinks}
\usepackage{textcomp}

\usepackage{tabstackengine}
\setstackEOL{\cr}

\usepackage{float}

\usepackage[utf8]{inputenc}

\usepackage{gensymb}
\usepackage{caption}
\usepackage{subcaption}
\usepackage[ruled,vlined]{algorithm2e}

\DeclareMathOperator*{\argmin}{arg\,min}
\usepackage{mathtools}
\usepackage{tikz}
\usepackage{environ}
\usetikzlibrary{positioning}
\usepackage{enumerate}
\usepackage{enumitem}

\usepackage{tabularx,ragged2e}
\newcolumntype{C}{>{\Centering\arraybackslash}X} 
\newcolumntype{L}{>{\arraybackslash}X} 

\usepackage{amsthm}

\newtheorem{theorem}{Theorem}

\newtheorem{proposition}{Proposition}
\newtheorem{problem}{Problem}
\newtheorem{assumption}{Assumption}

\newtheorem{definition}{Definition}

\newtheorem{remark}{Remark}

\newcommand{\Obs}[1]{\mathcal{O}^{#1}}
\newcommand{\dist}{\textnormal{dist}}

\begin{document}
\title{Autonomous Navigation with Convergence Guarantees in Complex Dynamic Environments}
\author{Albin Dahlin and Yiannis Karayiannidis, \IEEEmembership{Member, IEEE}
\thanks{This work was
supported in part by the Chalmers AI Research Centre (CHAIR) and AB Volvo
through the project AiMCoR and in part by the ELLIIT Strategic Research Area.}
\thanks{A. Dahlin is with the Department of Electrical Engineering, Chalmers University of Technology, SE-412 96 Gothenburg, Sweden 
        {\tt\small albin.dahlin@chalmers.se}}
\thanks{Y. Karayiannidis is with the Department of Automatic Control, LTH, Lund University,  SE-221 00 Lund, Sweden. Y. K. is a member of the ELLIIT Strategic Research Area at Lund University.
        {\tt\small yiannis@control.lth.se}.}}

\maketitle
\begin{abstract}
    This article addresses the obstacle avoidance problem for setpoint stabilization and path-following tasks in complex dynamic 2D environments that go beyond conventional scenes with isolated convex obstacles. A combined motion planner and controller is proposed for setpoint stabilization that integrates the favorable convergence characteristics of closed-form motion planning techniques with the intuitive representation of system constraints through Model Predictive Control (MPC). The method is analytically proven to accomplish collision avoidance and convergence under certain conditions, and it is extended to path-following control. Various simulation scenarios using a non-holonomic unicycle robot are provided to showcase the efficacy of the control scheme and its improved convergence results compared to standard path-following MPC approaches with obstacle avoidance.
\end{abstract}

\section{Introduction}
\label{sec:intro}
Setpoint stabilization, which entails driving a system to a specified goal state, and path-following, which aims to follow a predefined path as closely as possible, are common tasks in autonomous agent applications, involving autonomous mobile vehicles \cite{strands_project_17}, drones \cite{cai_etal_14}, autonomous surface vessels \cite{mazenc_etal_02}, and robotic manipulators \cite{gill_etal_13}. As autonomous agents, or robots, are increasingly employed in dynamically changing environments, the need for sensor-based motion controllers able to react to unforeseen circumstances is prominent. 
To achieve successful online navigation in such environments, a key aspect is to adaptively modify the constraints imposed by the robot's surroundings, possibly involving the presence of moving obstacles. A vast part of the literature in online obstacle avoidance are based either on closed-loop or optimization based control solutions, where specific requirements on the obstacle shapes are imposed and cases of intersecting obstacles are ignored. However, closely positioned obstacles are frequently perceived as intersecting, e.g., when inflation is used to account for robot radius or safety margins, or in case of perception uncertainties.
Breaking the conditions of disjoint obstacles yields local minima, jeopardizing convergence to the desired goal or path. In this work we combine the convergence properties of closed-form Dynamical Systems (DS) and the intuitive encoding of system constraints for Model Predictive Control (MPC) to propose a holistic control solution with guaranteed convergence also in scenarios of nontrivial obstacle constellations.

\subsection{Related Work}
A popular motion planning paradigm is the use of sampling-based approaches, such as probabilistic road maps \cite{kavraki_etal_96} and rapidly exploring randomly trees \cite{lavalle_kuffner_00}. These are in their original forms not suitable for online motion planning and various methods to reduce computational complexity have been proposed \cite{marble_bekris_13,ichter_etal_20,gammell_etal_14,yang_etal_14}. A more computationally efficient strategy is to construct closed form DS that possess desirable stability and convergence properties, eliminating the need to find a complete path at each iteration. Specifically, artificial potential fields \cite{khatib_85}, repelling the robot from the obstacles, have become popular \cite{ginesi_etal_19,stavridis_etal_17}. 
A drawback of the additive potential fields is the possible occurrences of local minimum other than the goal point. To address this problem, navigation functions \cite{rimon_koditschek_92,loizou_11,kumar_etal_22} and harmonic potential fields have emerged \cite{feder_slotine_97,huber_etal_19,huber_etal_22}. A repeated assumption in the aforementioned works enabling the proof of (almost) global convergence is the premise of the environment being a \textit{disjoint star world} (DSW), i.e. all obstacles are starshaped and mutually disjoint\footnote{See Section \ref{sec:starsets} for complete definition.}. 
However, intersecting obstacles are frequently occurring, e.g., when modelling complex obstacles as a combination of several simpler shapes, or when the obstacle regions are padded to take robot radius or safety margins into account. To handle intersecting circular obstacles, \cite{daily_bevly_08} proposed weighted average of harmonic functions, but unwanted local minima still occurred and the authors recommended to keep the number of combined obstacles low.
In \cite{dahlin_karayiannidis_23_1} we presented a method, here referred to as ModEnv$^\star$, which modulates the robot environment to obtain a DSW to extend the applicable scenarios where the aforementioned DS methods achieve convergence properties. The approach was limited to the case with a robot operating in the full Euclidean space and no conditions for successful generation of a DSW was provided. 
In later years, MPC has become increasingly popular, where the obstacle regions (or approximation of the regions) are typically explicitly expressed in the optimization problem \cite{schulman_etal_14,zhang_etal_21,hermans_etal_21}. Compared to the closed form control laws, MPC allows for an easy encoding of the system constraints and formulation of desired behaviors, such as smooth control input. However, due to the receding horizon nature of the MPC, convergence guarantees are not provided. Specifically, in environments with large or intersecting obstacles, the MPC solution may lead to local attractors at obstacle boundaries. A simultaneous path planning and tracking framework was proposed in \cite{ji_etal_16}, combining potential fields and MPC. The method however relies on additive potential fields which may introduce local attractors in the case of closely positioned obstacles. 
In \cite{dahlin_karayiannidis_23_2} we presented a motion control scheme for setpoint stabilization with collision avoidance consisting of three main components: environment modification into a DSW, DS-based generation of a receding horizon reference path (RHRP), and an MPC to compute admissible control inputs to drive the robot along the RHRP. 
Whereas collision avoidance is ensured, no guarantees for convergence were provided. Convergence may be inhibited by two situations; 1) the modified environment is not a DSW such that convergence guarantees for the DS method are lost, 2) the MPC solution does not provide a movement of the robot along the RHRP due to limited control horizon and robot constraints. 

Various path-following techniques considering obstacle avoidance have been presented. In \cite{lapierre_etal_07}, a backstepping approach was presented for a unicycle type where obstacle avoidance was obtained through the Deformable Virtual Zone principle, path-following using Line-of-Sight for Unmanned Surface Vessels was in \cite{moe_pettersen_16} adapted to obtain collision avoidance, and vector fields was constructed for use in a Unmanned Aerial Vehicle in \cite{wilhelm_clem_19}. 
As for setpoint stabilization, there has been a growing research focus for path-following control based on MPC where obstacles are encoded directly as constraints \cite{howard_etal_10,brito_etal_19,zube_15,arbo_etal_17}. As stated above, these approaches may however lead the robot to full stop in occasions of intersecting obstacles where the objectives of path-following and obstacle avoidance conflict. To mitigate the risk of being trapped at obstacle boundaries, the optimization problem was relaxed in \cite{sanchez_etal_21} by introduction of an auxiliary reference. Obstacle avoidance is however attained based on additive potential fields and the problem of undesired local attractors is not resolved.

\subsection{Contribution}
In this work, we expand upon the control scheme introduced in \cite{dahlin_karayiannidis_23_2} to enable the derivation of convergence properties and to facilitate its implementation within confined workspaces. As convergence rely on environment modification into a DSW, we first derive sufficient conditions to obtain a DSW for ModEnv$^{\star}$ (Algorithm 2 in \cite{dahlin_karayiannidis_23_1}). Additionally, the method is enhanced to treat also the case of confined workspaces.
Moreover, the control scheme is extended to obtain path-following behavior with obstacle avoidance. 

In all, the main contributions are:
\begin{itemize}
    \item Extension of ModEnv$^\star$ to allow for confined workspaces and derivation of sufficient conditions to successfully obtain a DSW.
    \item A setpoint stabilizing control scheme for collision avoidance with derivation of sufficient conditions for convergence.
    \item A combined path-following and collision avoidance control scheme that integrates a DS motion planning approach with MPC, making the optimal control problem independent of workspace complexity.
\end{itemize}

\section{Preliminaries}
\subsection{Notation}
Let $A=\{A^1,A^2,...\}, A^i\in \mathbb{R}^d$ be a collection of sets. The union and intersection of $A$ are denoted by $A_{\cup}=\bigcup_{A^i\in A}A^i$ and $A_{\cap}=\bigcap_{A^i\in A}A^i$, respectively. If all sets $A^i\in A$ are starshaped, the kernel intersection is denoted by $\textnormal{ker}_{\cap}(A)=\bigcap_{A^i\in A} \textnormal{ker}(A^i)$.
For convenience, an improper use of the Minkowski sum, $\oplus$, will be applied as follows: $A\oplus B = \{A^i\oplus B\}_{\forall A^i\in A}$, given $B\in \mathbb{R}^d$. The closest distance between two sets, $A^1$ and $A^2$, is denoted by $\dist(A^1,A^2)$. $\mathbb{B}(a, b)$ and $\mathbb{B}[a, b]$ are the open and closed balls of radius $b$ centered at $a$, respectively. The line segment from point $x$ to point $y$ is denoted by $l[x,y]$. 
A robot workspace $\mathcal{W}\subset\mathbb{R}^2$ and a collection of obstacles $\mathcal{O} = \{\Obs{1}, \Obs{2}, ...\}$ in $\mathbb{R}^2$ are jointly called the robot environment, denoted by $E=\{\mathcal{W},\mathcal{O}\}$. The corresponding free set is denoted by $\mathcal{F}=\mathcal{W}\setminus\mathcal{O}_{\cup}$.

\subsection{Starshaped Sets and Star Worlds}
\label{sec:starsets}
A set $A$ is \textit{starshaped with respect to} (w.r.t.) $x$ if for every point $y\in A$ the line segment $l[x,y]$ is contained by $A$. The set $A$ is said to be \textit{starshaped} if it is starshaped w.r.t.  some point, i.e. $\exists x$ s.t. $l[x,y] \subset A, \forall y \in A$. The set of all such points is called the \textit{kernel of $A$} and is denoted by $\textnormal{ker}(A)$, i.e. $\textnormal{ker}(A) = \{x\in A : l[x,y] \subset A, \forall y\in A\}$. For any convex set $A$ we have $\textnormal{ker}(A) = A$. 
The set $A$ is \textit{strictly starshaped w.r.t. $x$} if it is starshaped w.r.t. $x$ and any ray emanating from $x$ crosses the boundary only once. We say that $A$ is strictly starshaped if it is strictly starshaped w.r.t. some point.

The robot environment $E=\{\mathcal{W},\mathcal{O}\}$ is said to be a \textit{star world} if all obstacles are strictly starshaped, and the workspace is strictly starshaped or the full Euclidean space. 
A \textit{disjoint star world} (DSW) refers to a star world where all obstacles are mutually disjoint and where any obstacle which is not fully contained in the workspace has a kernel point in the exterior of the workspace, as exemplified in Fig. \ref{fig:dsw_example}. 
For more information on starshaped sets and star worlds, see \cite{hansen_etal_20} and \cite{dahlin_karayiannidis_23_1}.

\subsection{Obstacle Avoidance for Dynamical Systems in Star Worlds}
\label{sec:soads}
Given a star world, $E$, collision avoidance can be achieved using a DS approach \cite{huber_etal_22} with dynamics:
\begin{equation}
    \label{eq:ds_obs_avoidance}
    \dot{r} = \eta(r,r^g,E) = M(r,E)(r^g-r),
\end{equation}
where $r$ is the current robot position and $r^g\in \mathcal{F}$ is the goal position. $M(\cdot)$ is a modulation matrix used to adjust the attracting dynamics to $r^g$ based on the obstacles tangent spaces.
Convergence to $r^g$ is guaranteed for a trajectory following \eqref{eq:ds_obs_avoidance} from any initial position, $r^0\in\mathcal{F}$, if $E$ is a DSW and no obstacle center point is contained by the line segment $l[r^0,r^g]$. 
For more information, see \cite{huber_etal_19,huber_etal_22}.

\section{Problem Formulation}
Consider an autonomous agent with dynamics
\begin{equation}
\label{eq:robot_model}
\begin{split}
    \dot{x}(t) &= f(x(t),u(t))\\
    p(t) &= h(x(t)),
\end{split}
\end{equation}
where $x \in\mathcal{X}\subset \mathbb{R}^n$ is the robot state, $p\in \mathbb{R}^2$ is the robot position and $u\in \mathcal{U}\subset \mathbb{R}^m$ is the control signal.
It is assumed that there exists a control input such that the robot does not move, i.e. $\exists u'\in \mathcal{U} \textnormal{ s.t. } f(x,u') = 0,\ \forall x\in\mathcal{X}$. 
The robot is operating in a dynamic environment, $E(t)=\{\mathcal{W}(t),\mathcal{O}(t)\}$, where each obstacle is either convex or a simple polygon, and the workspace is either strictly starshaped or the full Euclidean plane. The free space - the collision-free robot positions - is then given as $\mathcal{F}(t) = \mathcal{W}(t) \setminus \mathcal{O}_{\cup}(t)$.
\begin{remark}
Although $\mathcal{O}$ formally contains only polygons and convex shapes, the formulation allows for more general complex obstacles as intersections are allowed. In particular, any shape can be described as a combination of several polygon and/or convex regions.
\end{remark}
The objective is to find a control policy that enforces the robot to stay in the free set at all times while driving the robot 1) to a specified goal position, or 2) closely along a predefined path, $\Gamma$. The path $\Gamma$ is a parametrized regular curve
\begin{equation}
\label{eq:Gamma}
    \Gamma = \{p \in \mathbb{R}^2 : \theta \in [0, \theta^g] \rightarrow p = \gamma(\theta)\},
\end{equation}
where the scalar variable $\theta$ is called the path parameter, and $\gamma : \mathbb{R}\rightarrow \mathbb{R}^2$ is a natural parametrization of $\Gamma$.
Formally, the problems are defined as follows:

\begin{problem}[Setpoint stabilization with obstacle avoidance]
\label{problem:setpoint_tracking}
Given the robot dynamics \eqref{eq:robot_model}, the environment $E(t)$, and a goal position $p^g\in \mathbb{R}^2$, design a control scheme that computes $u(t)\in\mathcal{U}$ such that robot stays in the free space at all times and converges to the goal. That is, $p(t)\in \mathcal{F}(t) \forall t$, and $\lim_{t\rightarrow \infty} p(t) = p^g$.
\end{problem}
\begin{problem}[Path-following with obstacle avoidance]
\label{problem:path_following}
Given the robot dynamics \eqref{eq:robot_model}, the environment $E(t)$, and reference path $\Gamma$, design a controller that computes $u(t)\in\mathcal{U}$ and $\theta(t)\in [0, \theta^g]$ such that robot stays in the free space at all times, moves in forward direction along $\Gamma$, and converges to $\gamma(\theta^g)$. 
\end{problem}

In the following sections, we will omit the time notation for convenience unless some ambiguity exists.

\section{Guaranteed DSW Generation}
\label{a:kernel_selection}
The obstacle avoidance approach layed out in Section \ref{sec:setpoint_stabilization} is dependent on ModEnv$^{\star}$ presented in \cite{dahlin_karayiannidis_23_1}. No guarantee for successfully obtaining a DSW was however derived, preventing convergence guarantees to be established for the proposed control scheme. Moreover, the workspace was assumed to be the full Euclidean space, ignoring situations where the workspace is bounded. Here, we extend ModEnv$^{\star}$ to address both these issues. Specifically, the kernel point selection is adjusted according to Algorithm \ref{alg:kernel_point_selection} in Appendix \ref{a:kernel_selection}. 
To declare a sufficient condition for DSW generation, the following definition of a \textit{DSW equivalent} set is established.
\begin{definition}[DSW equivalent]
A star world is DSW equivalent if the set, $Cl$, formed by partitioning $\mathcal{O}$ into mutually disjoint clusters of obstacles, satisfies
\begin{enumerate}[label=\roman*)]
    \item the obstacles in each cluster have intersecting kernels, 
    \begin{equation}
    \label{eq:dsw_kernel}
      \textnormal{ker}_{\cap}(cl) \neq \emptyset,\ \forall cl\in Cl,  
    \end{equation}
    \item any cluster not completely inside the workspace has an intersecting kernel region that does not fall entirely within the workspace, 
    \begin{equation}
    \label{eq:dsw_workspace_kernel}
    cl_{\cup}\not\subset\mathcal{W}\Rightarrow\textnormal{ker}_{\cap}(cl) \not\subset \mathcal{W},\quad \forall cl\in Cl.      \end{equation}
\end{enumerate}
\end{definition}

With the adjusted kernel point selection, a sufficient condition to establish a DSW can be presented as stated in the following theorem.
\begin{theorem}[Guaranteed DSW generation]
\label{theorem:dsw_success}
Consider a DSW equivalent environment with free space $\mathcal{F}$, a robot position, $p\in\mathcal{F}$, and a goal position, $p^g\in\mathcal{F}$. The environment, $\{\mathcal{W},\mathcal{O}^{\star}\}$, resulting from ModEnv$^{\star}$ with kernel point selection as in Algorithm \ref{alg:kernel_point_selection} is a DSW with $\mathcal{O}^{\star}_{\cup}=\mathcal{O}_{\cup}$.
\end{theorem}

\textit{Proof:} See Appendix \ref{a:dsw_success}.

An example of a DSW equivalent scene and the corresponding DSW is shown in Fig. \ref{fig:dsw_example}. 
\def\figscale{0.48}
\begin{figure}[ht!]
    \begin{subfigure}[t]{\figscale\linewidth}
        \includegraphics[width=\linewidth]{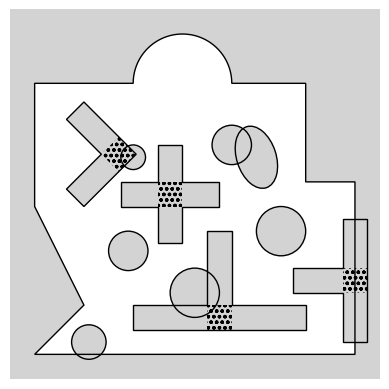}
    \end{subfigure}
    \hfill
    \begin{subfigure}[t]{\figscale\linewidth}
        \includegraphics[width=\linewidth]{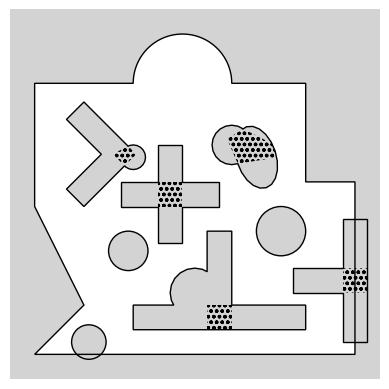}
    \end{subfigure}
    \caption{A DSW equivalent environment (left) and the corresponding DSW (right). The kernels for all non-convex obstacles are shown as dotted .}
    \label{fig:dsw_example}
\end{figure}

\section{Setpoint Stabilization with Obstacle Avoidance}
\label{sec:setpoint_stabilization}
In this section, a control scheme for setpoint stabilization with obstacle avoidance, addressing Problem \ref{problem:setpoint_tracking}, is proposed. The scheme is divided into four main components as illustrated in Fig. \ref{fig:block_scheme}. The environment is modified to form a DSW, $E^{\star}$, where any free point has an appropriately selected minimum clearance, $\rho$, to the obstacles and workspace boundary in $E$ (Section \ref{sec:workspace_modification}). This enables generation of a \textit{receding horizon reference path} (RHRP), $\mathcal{P}$, based on \eqref{eq:ds_obs_avoidance} to ensure collision clearance and convergence to the goal (Section \ref{sec:rhrp}). A control sequence, $z^*$, is computed using an MPC which yields a robot movement along the RHRP within the specified clearance to ensure collision avoidance (Section \ref{sec:mpc}). To provide guaranteed forward motion, initial movement of the reference is enforced in the MPC formulation. As a consequence, there may be occasions where the MPC problem is infeasible for non-holonomic robots. To handle this, a backup control law is formulated (Section \ref{sec:sbc}) and a control law scheduler is defined (Section \ref{sec:scl}) yielding the control policy, $\mu$, applied by the controller over the following sampling period. The complete control scheme is outlined in Section \ref{sec:setpoint_stabilization_scheme} where collision avoidance and convergence properties are also analyzed.

\begin{figure}[ht!]
    \centering
    \resizebox{\linewidth}{!}{
        \begin{tikzpicture}
    \node [draw,
        text width=8em, fill=blue!20, text centered,
        minimum height=5em, rounded corners
    ]  (env_mod) at (0,0) {Environment modification};
    \node [draw,
        text width=8em, fill=blue!20, text centered, 
        minimum height=5em, rounded corners,
        right=5em of env_mod
    ] (rhrp) {Receding horizon\\ reference path};
    \node [draw,
        text width=8em, fill=blue!20, text centered, 
        minimum height=5em, rounded corners,
        right=2em of rhrp
    ] (scl) {Switching control law};
    \node [draw,
        text width=8em, fill=blue!20, text centered, 
        minimum height=5em, rounded corners,
        below=2em of scl
    ] (mpc) {MPC};
    \node [draw,
        text width=5em, fill=white!20, text centered, 
        minimum height=3em, rounded corners,
        below=8em of env_mod
    ]  (robot) {Robot};
    \node [draw,
        text width=5em, fill=brown!20, text centered, 
        minimum height=3em, rounded corners,
        right=5.4em of robot
    ] (controller) {Controller};
    \node [left=2em of env_mod](scene){};
    \node [above=2em of scl.130](O_r0_rg){};
    \node [above=4em of scl](r+_O+){};
    \node [right=1em of scl](pi){};
    \node at (mpc.195 -| robot) (x){};
    \node at (x -| controller) (x_extra){};
    \draw[-stealth] (scene.center) -- (env_mod.west) 
        node[near start,above]{$E, p^g$};
    \draw[-stealth] (env_mod.east) -- (rhrp.west) 
        node[midway,below,text width=5em,text centered](env_mod_out){$E^{\star}, r^0, r^g$};
    \draw[-stealth] (rhrp.east) |- (scl.west) node[near end,above](P){$\mathcal{P}$};
    \draw[-stealth] (mpc.north) -| (scl.south) node[near end,right]{$z^*$};
    \draw[-] (scl.east) -- (pi.center) node[near end, above]{$\mu$};
    \draw[-stealth] (pi.center) |- (controller.east);
    \draw[-stealth] (controller.west) -- (robot.east) node[midway,above]{$u=\mu(x)$};
    \draw[-] (robot.north) -- (x.center) node[near start, left]{$x$};
    \draw[-stealth] (x.center) -- (env_mod.south);
    \draw[-] (x.center) -- (x_extra.center);
    \draw[-stealth] (P.south) |- (mpc.165);
    \draw[-stealth] (x_extra.center) -- (mpc.195);
    \draw[-stealth] (x_extra.center) -- (controller.north);
    \draw[-stealth] (env_mod.300) |- (mpc.west) node[near start,right]{$\rho$};
    \draw[-] (scl.north) -- (r+_O+.center) node[near start, right]{$r^+, \mathcal{O}^+$};
    \draw[-stealth] (r+_O+.center) -| (env_mod.north);
    \draw[-] (env_mod_out.north) |- (O_r0_rg.center);
    \draw[-stealth] (O_r0_rg.center) -- (scl.130);
    \end{tikzpicture}
    }
    \caption{Proposed motion control scheme for setpoint stabilization.}
    \label{fig:block_scheme}
\end{figure}
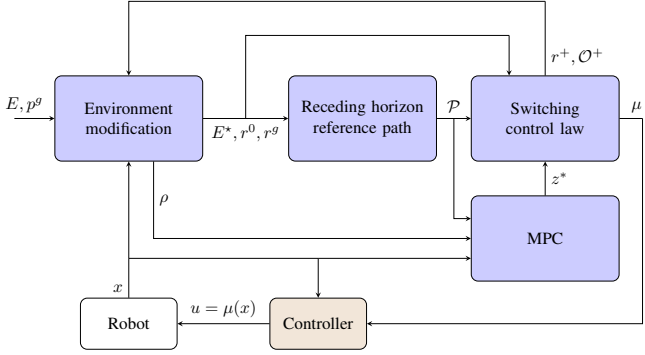

\subsection{Environment Modification}
\label{sec:workspace_modification}
The proposed method relies on generating the RHRP with a (time-varying) minimum clearance, $\rho$, to all obstacles using the DS approach \eqref{eq:ds_obs_avoidance}. To this end, the clearance environment $E^{\rho}=\{\mathcal{W}^{\rho},\mathcal{O}^{\rho}\}$ is defined, where $\mathcal{W}^{\rho} = \mathcal{W} \ominus \mathbb{B}[0, \rho]$ and $\mathcal{O}^{\rho} = \mathcal{O} \oplus \mathbb{B}[0, \rho]$, with corresponding clearance space $\mathcal{F}^{\rho} = \mathcal{W}^{\rho} \setminus \mathcal{O}^{\rho}_{\cup}$. 
As stated in Section \ref{sec:soads}, any star world is positively invariant for the dynamics \eqref{eq:ds_obs_avoidance} and convergence to a goal position is guaranteed for a DSW. 
Since $E^{\rho}$ may include both intersecting and non-starshaped obstacles, it provides none of the aforementioned guarantees. The objective of the environment modification is therefore to find a DSW $E^{\star}=\{\mathcal{W}^{\rho},\mathcal{O}^{\star}\}$ with corresponding free set $\mathcal{F}^{\star}\subset\mathcal{F}^{\rho}$, as well as initial and goal positions, $r^0\in\mathcal{F}^{\star}$ and $r^g\in\mathcal{F}^{\star}$, for the RHRP. A procedure to specify $\rho$ and to compute $E^{\star}$, $r^0$ and $r^g$ is given in Algorithm \ref{alg:obstacle_transformation} and the steps are elaborated below.

\LinesNumbered
\begin{algorithm}[ht!]
\caption{Environment modification}
\label{alg:obstacle_transformation}
\SetKwRepeat{Do}{do}{while}
\SetKwInOut{Input}{Input}
\SetKwInOut{Output}{Output}
\SetKwInOut{Parameters}{Parameters}
\SetKwInOut{Init}{Init}{}{}
\Parameters{$\alpha\in(0,1)$, $\bar{\rho}\in\mathbb{R}^+$}
\Input{$\mathcal{W}$, $\mathcal{O}$, $p^g$, $p$, $r^+$}
\Output{$\mathcal{W}^{\rho}$, $\mathcal{O}^{\star}, r^0, r^g, \rho$}
\Init{$E^+ \gets \{\mathcal{W}^{\bar{\rho}},\emptyset\}$}
    \eIf{$p\in\mathcal{C}^{\bar{\rho}}$}{ \label{l:rho_base_valid}
        $\rho \gets \bar{\rho}$\;
    }{
        $\rho \gets \alpha \dist(\partial\mathcal{F},p)$\;
    }
    $\{\mathcal{W}^{\rho},\mathcal{O}^{\rho}\} \gets \{\mathcal{W} \ominus \mathbb{B}[0,\rho], \mathcal{O} \oplus \mathbb{B}[0,\rho]\}$\; \label{l:environment_rho}
    $\mathcal{P}^0 \gets \mathcal{F}^{\rho} \cap \mathbb{B}[p,\rho]$\; \label{l:init_ref_set}
    $r^0 \gets \argmin_{r^0\in \mathcal{P}^0} \lVert r^0 - r^+ \lVert_2$\; \label{l:r0}
    $r^g \gets \argmin_{r^g\in \mathcal{F}^{\rho}} \lVert r^g - p^g \lVert_2$\; \label{l:rg}
    \eIf{$r^0\in \mathcal{F}^{+}$ \textnormal{\textbf{and}} $r^g\in \mathcal{F}^+$ \textnormal{\textbf{and}} $\mathcal{F}^+ \subset \mathcal{F}^{\rho}$}{ \label{l:starify_start}
        $\mathcal{O}^{\star} \gets \mathcal{O}^+$\;
    }{
        Compute $\mathcal{O}^{\star}$ using ModEnv$^{\star}$ with $\mathcal{W}^{\rho}$, $\mathcal{O}^{\rho}$, $r^0$ and $r^g$ as input\;
    } \label{l:starify_end}
    \ForEach{$\Obs{j} \in \mathcal{O}^{\star}$ \label{l:conv_start}}{
        \If{$CH(\Obs{j}) \cap \left\{r^0 \cup r^g \cup \left(\mathcal{O}^{\star}\setminus\Obs{j}\right)_{\cup} \right\} = \emptyset$}{
            $\Obs{j} \gets CH(\Obs{j})$\;
        }
    } \label{l:conv_end}
    \eIf{$E^{\star}$ \textnormal{is a DSW}\label{l:estar_dsw}}{
        $E^+ \gets E^{\star}$\;
    }{
        $E^+ \gets \{\mathcal{W}^{\bar{\rho}},\emptyset\}$\; \label{l:nominal_e+}
    }
\end{algorithm}

\subsubsection*{Initial and goal reference position selection (lines \ref{l:init_ref_set}-\ref{l:rg})}
The initial reference position, $r^0$, is chosen as the point closest to an input candidate, $r^+$,  within the initial reference set $\mathcal{P}^0 = \mathcal{F}^{\rho} \cap \mathbb{B}[p, \rho]$. In this way, the distance from $r^0$ to any obstacle and workspace boundary is greater than $\rho$, while the distance to the robot is less or equal to $\rho$. As specified in Section \ref{sec:scl}, $r^+$ is appropriately selected along the previously computed RHRP. In particular, $r^+$ is chosen to stimulate a forward shift of the RHRP towards the goal, compared to the previous sampling instance.
The reference goal, $r^g$, is chosen as the point in $\mathcal{F}^{\rho}$ closest to $p^g$.

\subsubsection*{Clearance selection (lines \ref{l:rho_base_valid}-\ref{l:environment_rho})}
To have a valid initial reference position, $\rho$ is set to a strict positive value such that $\mathcal{P}^0$ is nonempty. This is done by utilizing the equivalence 
\begin{equation}
\label{eq:valid_rho}
    p\in \mathcal{C}^{\rho} \Leftrightarrow \mathcal{P}^0 \neq \emptyset,\quad \mathcal{C}^{\rho}=\mathcal{F}^{\rho}\oplus\mathbb{B}[0,\rho].
\end{equation}
For robot positions $p\notin\mathcal{C}^{\bar{\rho}}$, i.e. when the default selection $\rho=\bar{\rho}$ yields $\mathcal{P}^0 = \emptyset$, the clearance is reduced to $\rho=\alpha \dist(\partial\mathcal{F}, p)$ to ensure $p\in\mathcal{F}^{\rho}$ and thus $\mathcal{P}^0 \neq \emptyset$ according to \eqref{eq:valid_rho}. This is a conservative reduction of $\rho$ since larger values could in many cases be used while still obtaining $\mathcal{P}^0\neq\emptyset$. 
The procedure is illustrated Fig. \ref{fig:rho_selection}.

\begin{figure}[ht!]
    \begin{subfigure}[t]{0.36\linewidth}
        \centering
           \includegraphics[height=5cm,trim={0.5cm 0.5cm 0.5cm 0.5cm},clip]{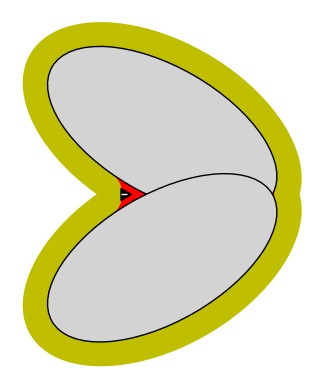}
         \caption{The default clearance $\rho=\bar{\rho}=0.2$ yields $\mathcal{P}^0=\emptyset$ since $p\notin\mathcal{C}^{\bar{\rho}}$.}
         \label{fig:invalid_rho_0}
    \end{subfigure}
    \hfill
    \begin{subfigure}[t]{0.6\linewidth}
        \begin{tikzpicture}[      
                every node/.style={anchor=south west,inner sep=0pt},
                x=1mm, y=1mm,
              ]   
             \node (fig2) at (63,0)
               {\includegraphics[height=5cm,trim={0.5cm 0.5cm 0.5cm 0.5cm},clip]{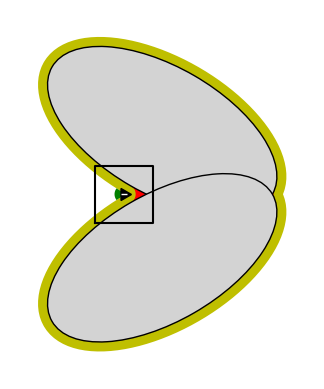}};
               \node (fig4) at (38,30)
               {\includegraphics[width=0.3\linewidth,trim={0.2cm 0.2cm 0.2cm 0.2cm},clip]{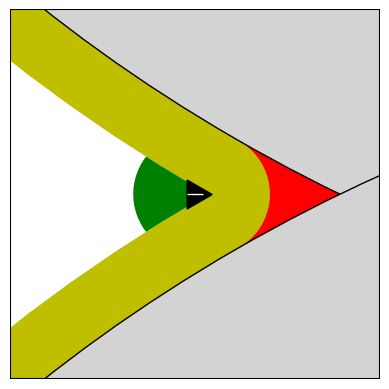}};
               \node (fig5) at (38,0)
               {\includegraphics[width=0.3\linewidth,trim={0.2cm 0.2cm 0.2cm 0.2cm},clip]{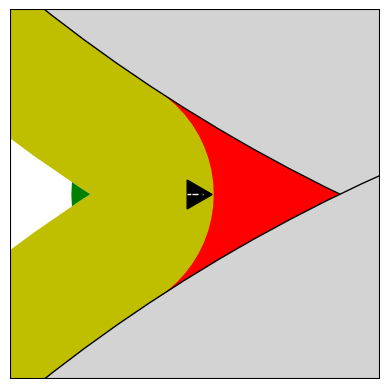}};
            \draw[densely dotted, line width=.8pt] (73.5, 21) -- (fig4.south west);
            \draw[densely dotted, line width=.8pt] (82, 29) -- (fig4.north east);
            \node at (45, 30.5) {\small $\rho=0.075$};
            \node at (46.7, 0.5) {\small $\rho=0.15$};
        \end{tikzpicture}
         \caption{The clearance is reduced to $\rho=\gamma \dist(\partial\mathcal{F},p)$ to ensure $\mathcal{P}^0\neq \emptyset$. The less conservative reduction $\rho=0.15$ would also be a valid selection since this still yields $\mathcal{P}^0\neq\emptyset$.}
         \label{fig:valid_rho_0}
    \end{subfigure}
    \caption{Two obstacles $\mathcal{O}$ (grey) with inflated regions $\mathcal{O}^{\rho}$ (yellow), the initial reference set $\mathcal{P}^0$ (green), the region $\mathcal{F}\setminus\mathcal{C}^{\rho}$ (red) corresponding to robot positions where $\mathcal{P}^0=\emptyset$, and the robot position $p$ (black triangle).}
    \label{fig:rho_selection}
\end{figure}
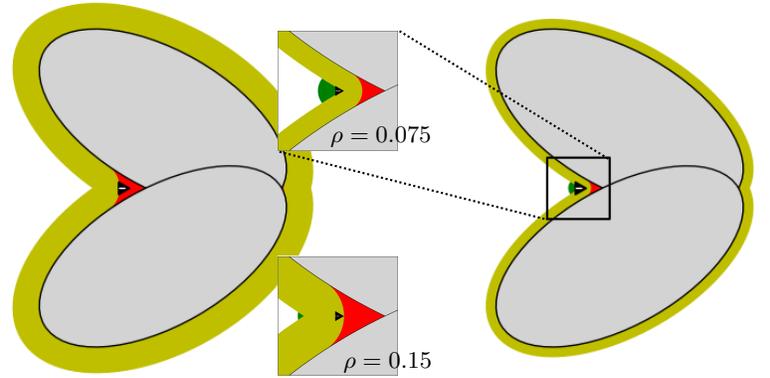

\subsubsection*{Establishment of a DSW (lines \ref{l:starify_start}-\ref{l:starify_end}, \ref{l:estar_dsw}-\ref{l:nominal_e+})}
To obtain a DSW such that $\mathcal{F}^{\star}\subset\mathcal{F}^{\rho}$, $r^0\in\mathcal{F}^{\star}$ and $r^g\in\mathcal{F}^{\star}$, a first attempt is to use a candidate environment, $E^+$. If this does not satisfy the conditions for $\mathcal{F}^{\star}$, the obstacles are computed using ModEnv$^\star$. 
As specified in lines \ref{l:estar_dsw}-\ref{l:nominal_e+}, $E^+$ is the previously computed $E^{\star}$ if it is a DSW.

\subsubsection*{Convexification (lines \ref{l:conv_start}-\ref{l:conv_end})}
To avoid unnecessary ``detours'' in concave obstacle regions, see \cite{dahlin_karayiannidis_23_2}, the generated obstacles, $\mathcal{O}^{\star}$, are made convex provided that the following conditions are not violated: 1) $r^0$ and $r^g$ remain exterior points of the obstacle, and 2) the resulting obstacle region does not intersect with any other obstacle. Due to these conditions, any DSW $E^{\star}$ remains a DSW also after convexification.

In addition to the revised kernel point selection in Algorithm \ref{alg:kernel_point_selection}, the environment modification is adjusted compared to \cite{dahlin_karayiannidis_23_2} in four points: 1) the clearance, $\rho$, is established through a single-step method, guaranteeing $\mathcal{P}^0\neq\emptyset$, in contrast to employing an iterative approach, 2) the approach applies also for confined workspaces, 3) the selection of the initial reference point, $r^0$, is based on an input candidate, $r^+$, rather than the robot position $p$, and 4) if possible, the environment $E^{\star}$ is reused instead of being recalculated. 
The first adjustment is a pure simplification of the algorithm, the second extends the applicability to bounded workspaces, whereas the two last are instrumental to obtain the convergence properties derived in Section \ref{sec:setpoint_stabilization_scheme}. Moreover, maintaining a constant environment $E^{\star}$ across sampling instances results in increased consistency over time for the vector field used to generate the RHRP. This, coupled with the effort to initialize the RHRP along the previously computed one, facilitates smoother transitions of the path between control sampling instances

\subsection{DS-based Receding Horizon Reference Path}
\label{sec:rhrp}
The RHRP is given as a parametrized regular curve
\begin{equation}
\label{eq:receding_reference_path}
    \mathcal{P} = \left\{r\in \mathbb{R}^2 : s \in [0, L] \rightarrow r(s) \right\}
\end{equation}
with $L=Tw^{\max}$. Here, $T$ is the MPC horizon described in Section \ref{sec:mpc} and $w^{\max} = \max_{u\in \mathcal{U},x\in\mathcal{X}}\lVert \frac{\partial h}{\partial x}(x)f(x,u) \rVert_2$ is the maximum linear speed which can be achieved by the robot. The mapping $r$ is given by the solution to the ODE
\begin{equation}
\label{eq:receding_reference_dynamics}
    \frac{dr(s)}{ds} = \bar{\eta}(r(s),r^g, E^{\star}),\quad r(0) = r^0,
\end{equation}
where $\bar{\eta}(\cdot)=\frac{\eta(\cdot)}{\lVert \eta(\cdot) \rVert_2}$ are the normalized dynamics in \eqref{eq:ds_obs_avoidance}. As the path is initialized in the star world $\mathcal{F}^{\star}$ and the dynamics are positively invariant in any star world, we have $\mathcal{P}\subset \mathcal{F}^{\star}\subset\mathcal{F}^{\rho}$. Thus, the tunnel-region $\mathcal{P}^{\rho}=\mathcal{P}\oplus\mathbb{B}[0, \rho]$ is in the free set, $\mathcal{F}$.

\subsection{Model Predictive Controller}
\label{sec:mpc}
In \cite{dahlin_karayiannidis_23_2}, an MPC is used to compute a control input driving the robot along the RHRP. Whereas collision avoidance is proven, local attractors away from the goal may arise in the workspace depending on control horizon and robot constraints. To improve attracting behavior towards the goal and derive convergence conditions, we here introduce an \textit{enforced initial forward motion} of the reference position resulting in the following MPC.

Adhering the path-following MPC framwork \cite{faulwasser_findeisen_16}, the system state and input are augmented with path coordinate, $s\in [0,L]$, and path speed, $w\in [0,w^{\max}]$, respectively. This embeds the reference trajectory $r(s)$ as part of the optimization problem. The bounds on path variables ensure valid mapping $r(s)$ for all admissible $s$ and that the reference moves in forward direction along $\mathcal{P}$ with a reference speed, $\lVert\dot{r}\rVert_2=\lVert\frac{dr(s)}{ds}w\rVert_2$, less or equal to the maximum linear speed of the robot, $w^{\max}$. Similar to the tunnel-following MPC \cite{vanduijkeren_19}, a constraint is imposed on the tracking error, $\varepsilon(\tau) = \lVert r(s(\tau))-h(\bar{x}(\tau)) \rVert_2$ such that the robot position is in a $\rho$-neighborhood of the reference position\footnote{In contrast to \cite{vanduijkeren_19} we apply strict, and not soft, constraints on the tracking error. This can be done and still ensure existence of solution from the design of the reference path. In particular, since $r(0)\in \mathcal{P}^0\subset \mathbb{B}[p,\rho]$.}.
As standard in the MPC framework, the control variables are piecewise constant over a sampling interval $\Delta t$ and are computed over a horizon $T=N\Delta t$, with $N\in\mathbb{N}^+$.
The optimization problem for the MPC to find  the control sequence, $z=\{\bar{u}_i,\ w_i : i\in[0,1,..,N-1]\}$, is proposed as
\begin{subequations}
\begin{align}
& & &\min_{z} \int_0^T
-c_w w(\tau) + c_e\varepsilon(\tau) d\tau + J(z) \label{eq:mpc_cost}\\
&\text{subject to} & &\nonumber\\
& \tau\in[0,T]:&  & \dot{\bar{x}}(\tau) = f(\bar{x}(\tau),\bar{u}(\tau)),\ \bar{x}(0)=x(t) \label{eq:mpc_robot_dyn} \\
& &  & \dot{s}(\tau) = w,\ s(0)=0, \label{eq:mpc_s_dyn} \\
&  &      & \bar{x}(\tau) \in \mathcal{X},\ \bar{u}(\tau) \in \mathcal{U}, \label{eq:mpc_robot_con}\\
&  &      & s(\tau) \in [0,L],\ w(\tau) \in [0,w^{\max}], \label{eq:mpc_s_con}\\
&  &      & \bar{u}(\tau) = \bar{u}_i,\ w(\tau)=w_i,\ 
 i=\left\lfloor\frac{\tau}{\Delta t}\right\rfloor, \label{eq:mpc_pwc}\\
&  &      & \varepsilon(\tau) \leq \rho,
\label{eq:mpc_error}\\
&  &      & w_0 \geq \frac{\lambda\rho}{\Delta t}. \label{eq:mpc_w0_con}
\end{align}
\label{eq:mpc}
\end{subequations}
Here, $\lfloor \cdot \rfloor$ is the floor function, and the notation $\bar{x}$ and $\bar{u}$ is used to denote the internal variables of the controller and distinguish them from the real system variables. The scalars $c_w>0$, $c_e>0$ and $\lambda\in (0,1)$ are tuning parameters, and $J(z)$ is a regularization term for the control input which can be tailored for the robot at hand, if desired. To ensure that the upper and lower bound on $w_0$ do not conflict, the relationship $\lambda\bar{\rho}\leq w^{\max}\Delta t$ must be satisfied. The inclusion of enforced initial forward motion of the reference position \eqref{eq:mpc_w0_con} is key when deriving the convergence properties in Section \ref{sec:setpoint_stabilization_scheme}.

\subsection{Stabilizing Backup Controller}
\label{sec:sbc}
It can be shown that the MPC problem \eqref{eq:mpc} without constraint \eqref{eq:mpc_w0_con} is feasible at all times by following the proof of Theorem 1 in \cite{dahlin_karayiannidis_23_2}. With the constraint \eqref{eq:mpc_w0_con}, existence of solution is however no longer guaranteed. Consider for instance the example with a non-holonomic robot in Fig. \ref{fig:no_solution} where the robot position is outside the region $\mathbb{B}[r(\lambda\rho),\rho]$. Depending on robot constraints, forcing an initial displacement such that $s(\Delta t)=\lambda\rho$ may lead to $\varepsilon(\Delta t)>\rho$, violating constraint \eqref{eq:mpc_error}. To handle these cases, a fallback strategy is here presented.

\begin{figure}[ht!]
    \centering
    \includegraphics[width=0.5\linewidth]{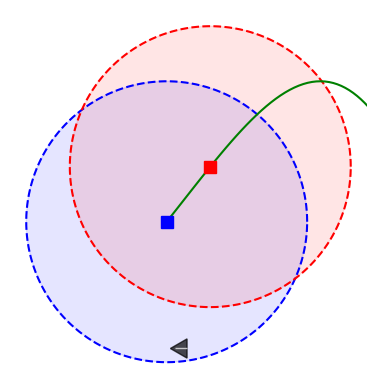}
    \caption{The initial part of $\mathcal{P}$ (green line) is depicted alongside $r^0$ (blue square) and $r(\lambda\rho)$ (red square) with corresponding regions $\mathbb{B}[r^0,\rho]$ (blue) and $\mathbb{B}[r(\lambda\rho),\rho]$ (red). If the robot constraints prohibits movement to achieve $p(\Delta t)\in\mathbb{B}[r(\lambda\rho),\rho]$, no solution to \eqref{eq:mpc} exists.}
    \label{fig:no_solution}
\end{figure}

Let $\mathcal{K} \subset \mathcal{X}\times\mathbb{R}^2 \rightarrow \mathcal{U}$ be a family of control laws such that any $\kappa \in \mathcal{K}$, when applying $u(\cdot)=\kappa(\cdot,r_0)$ in \eqref{eq:robot_model} given $r_0$,
\begin{enumerate}[label=\roman*)]
    \item renders the closed-loop error dynamics $\dot{e}$ asymptotically stable in the origin for the error $e(t)=p(t)-r_0$,
    \item does not allow the error to exceed its initial value, i.e.,\\ $\lVert e(t)\rVert_2 \leq \lVert e(t_0)\rVert_2,\ \forall t \geq t_0$.
\end{enumerate}
We will refer to any $\kappa\in\mathcal{K}$ as a \textit{stabilizing backup controller} (SBC).

\textbf{Unicycle example:}
Obviously, the SBC needs to be designed for the robot at hand, but an example is here presented for a unicycle robot kinematic model
\begin{equation}
\label{eq:unicycle}
    f(x,u) = \begin{bmatrix}v\cos\psi\\
    v\sin\psi\\
    \omega\end{bmatrix},\quad h(x) = \begin{bmatrix}p^x\\p^y\end{bmatrix},
\end{equation}
where $x=[p^x,p^y,\psi]^T$ are the Cartesian position and orientation of the robot, and $u=[v\ \omega]^T$ are the linear and angular velocities. The controller 
\begin{equation}
\label{eq:unicycle_kappa}
    \kappa(x,r^0) = \begin{bmatrix}
        v\\
        \omega
    \end{bmatrix} = \begin{bmatrix}
        -k_1\left(e^x\cos\psi + e^y\sin\psi\right)\\ 
        k_2 e_{\psi}
    \end{bmatrix},
\end{equation}
with $e_{\psi}=\left(\textnormal{atan2}(e^y, e^x) - \psi + \pi \right)$, $k_1>0$ and $k_2>0$ then satisfies the conditions for being an SBC given $\mathcal{U}=\mathbb{R}^2$. This can be derived using Barbalat's lemma with the positive semi-definite function $V=e^Te$ which has a negative semi-definite derivative, $\dot{V}=-k_1\left(e^x\cos\psi + e^y\sin\psi\right)^2$, under control law \eqref{eq:unicycle_kappa} \cite{siciliano_etal_08}.
To handle saturated controllers as well, upper bounds for $k_1$ and $k_2$ can be found to ensure $\kappa(x,r^0)\in\mathcal{U}$ by utilizing the fact $r^0\in\mathbb{B}[p,\rho]\subset\mathbb{B}[p,\bar{\rho}]$ and thus $\lVert e\rVert_2\leq \bar{\rho}$. Consider $\mathcal{U}=[v^{\min}, v^{\max}] \times [-\omega^{\max}, \omega^{\max}]$ with $v^{\min}<0<v^{\max}$ and $\omega^{\max}>0$. We have $\lvert v\rvert \leq k_1\left(\lvert e^x\rvert + \lvert e^y\rvert\right) \leq 2k_1\lVert e\rVert_2 \leq 2k_1\bar{\rho}$, and the upper bound $k_1\leq \frac{\min(-v^{\min},v^{\max})}{2\bar{\rho}}$ yields $v\in[v^{\min},v^{\max}]$. Assuming $e_{\psi}$ is evolving such that $e_{\psi}\in(-\pi,\pi]$, we have $\lvert \omega\rvert \leq k_2\pi$ and the upper bound $k_2\leq \frac{\omega^{\max}}{\pi}$ yields $\omega \in[-\omega^{\max},\omega^{\max}]$.

\subsection{Control Law Scheduler}
\label{sec:scl}
The control law $\mu$ is updated by a control law scheduler at a sampling interval $\Delta t$, such that it is constant over each period $t\in [t_k, t_{k+1})$ with $t_k=k\Delta t$, $k\in\mathbb{N}$. The control law switches between two modes, depending on feasibility of \eqref{eq:mpc} and if the RHRP is a singleton set\footnote{$r^0=r^g$ implies that the RHRP dynamics \eqref{eq:receding_reference_dynamics} are $\frac{dr}{ds}=0$ such that the RHRP is the singleton set $\mathcal{P}=\{r^g\}$.}, defined as
\begin{equation}
\begin{split}
    \textbf{MPC MODE}&: r^0\neq r^g \textnormal{ and \eqref{eq:mpc} is feasible}\\
    \textbf{SBC MODE}&: \textnormal{otherwise}.
\end{split}
\end{equation}
The control law scheduler determines the control law and the next initial reference position candidate according to the logic
\begin{align}
    \mu(\cdot) &= \begin{cases}
        \bar{u}^*_0, & \textbf{MPC MODE}, \\
        \kappa(\cdot, r_0), & \textbf{SBC MODE},
    \end{cases} 
    \label{eq:scl} \\
    r^+ &= \begin{cases}
        r(w^*_0\Delta t), & \textbf{MPC MODE}, \\
        r_0, & \textbf{SBC MODE},
    \end{cases}
    \label{eq:r_plus}
\end{align}
where $\bar{u}^{*}_{0}$ and $w^{*}_{0}$ are extracted from the optimal solution, $z^*$, of \eqref{eq:mpc}. The control input is hence constant over a sampling interval when \textbf{MPC MODE} is active while the feedback controller $\kappa$ is applied with $r^0$ as setpoint when \textbf{SBC MODE} is active. 
The initial reference candidate is specified along the RHRP. When in \textbf{MPC MODE}, a solution to the MPC problem exists and $r^+$ is chosen as the 1-step predicted reference position of the MPC solution. This encourages forward shift of the RHRP at the next sampling instance, while ensuring $r^+$ to stay in a $\rho$-neighborhood of the robot due to \eqref{eq:mpc_error}. When in \textbf{SBC MODE}, the control target is to realign the robot configuration to enable MPC feasibility at future sampling instances, and the next initial reference candidate is chosen as the current initial reference, suggesting no forward shift of the RHRP.

\subsection{Motion Control Scheme}
\label{sec:setpoint_stabilization_scheme}
The complete control scheme for setpoint stabilization is outlined in Algorithm \ref{alg:setpoint_control}. Although no information about the environment is used in the MPC formulation nor for the SBC, collision avoidance is achieved under the following assumptions as stated below by Theorem \ref{theorem:collision_avoidance}. This is obtained by ensuring a close tracking (with error less than $\rho$) of the path which is at least at a distance $\rho$ from any obstacle and workspace boundary.

\LinesNumbered
\begin{algorithm}
    \caption{Setpoint control scheme}
    \label{alg:setpoint_control}
    \SetKwRepeat{Do}{do}{while}
    \SetKwInOut{Input}{Input}
    \SetKwInOut{Output}{Output}
    \SetKwInOut{Parameters}{Parameters}
    \SetKwInOut{Init}{Init}
    \Parameters{$\bar{\rho}\in\mathbb{R}^+$, $\alpha\in(0,1)$, $\Delta t\in\mathbb{R}^+$, $N\in\mathbb{N}^+$, $\lambda\in(0,1)$, $c_w\in\mathbb{R}^+$, $c_e\in\mathbb{R}^+$, $\kappa\in\mathcal{K}$}
    \Input{$E$, $p^g$, $x$}
    \Output{$\mu(\cdot)$}
    \Init{$r^+ \gets p$}
    Compute $E^{\star},r^0,r^g,\rho$ using Algorithm \ref{alg:obstacle_transformation}\;
    Compute $\mathcal{P}$ according to \eqref{eq:receding_reference_path}-\eqref{eq:receding_reference_dynamics}\;
    Run solver for \eqref{eq:mpc}\;
    Update $\mu$ and $r^+$ according to \eqref{eq:scl}-\eqref{eq:r_plus}\;
\end{algorithm}

\begin{assumption}
\label{ass:slow_obstacles}
The obstacles move slow compared to the sampling frequency such that the obstacle positions are constant over a sampling period, i.e. $\mathcal{F}(t) = \mathcal{F}(t_k),\ \forall t\in[t_k, t_{k+1})$.
\end{assumption}
\begin{assumption}
\label{ass:not_aggressive_obstacles}
The obstacles do not actively move into a region occupied by the robot, such that the implication $p(t_k)\in\mathcal{F}(t_{k-1})
\Rightarrow p(t_k)\in\mathcal{F}(t_k)$ holds.
\end{assumption}

\begin{theorem}[Collision avoidance]
\label{theorem:collision_avoidance}
The trajectory for a robot with dynamics \eqref{eq:robot_model} and initial position $p(t_0)\in\mathcal{F}(t_0)$ following the motion control scheme in Algorithm \ref{alg:setpoint_control} is collision-free, i.e. $p(t) \in \mathcal{F}(t), \forall t\geq t_0$, if Assumptions \ref{ass:slow_obstacles}-\ref{ass:not_aggressive_obstacles} hold.
\end{theorem}

\textit{Proof:} See Appendix \ref{a:proof_collision_avoidance}.

While convergence properties cannot be stated for generic scenarios with dynamic obstacles (consider the case with iteratively opening and closing of two separated gaps in a room), it can be stated under the following assumptions. 

\begin{assumption}
\label{ass:static_obstacles}
There exists a time instance after which the environment is static, i.e. $\exists k^s\in\mathbb{N}\setminus\infty$ s.t. $E(t)=E_s=\{\mathcal{W}_s,\mathcal{O}_s\}, \forall t\geq t_{k^s}$.
\end{assumption}
\begin{assumption}
\label{ass:goal_distance}
The workspace boundary and all obstacles after time $t_{k^s}$ are at least at a distance $\bar{\rho}$ from the goal, i.e. $\dist(p^g,\{\partial\mathcal{W}_s,\mathcal{O}_s\}) \geq \bar{\rho}$.
\end{assumption}
Without loss of generality, we will in the following assume $k^s=0$.
Note that Ass. \ref{ass:slow_obstacles}-\ref{ass:static_obstacles} trivially hold for a static scene and Ass. \ref{ass:goal_distance} can easily be obtained by adjustment of $\bar{\rho}$ if $p^g\in\mathcal{F}_s$.
Under Ass. \ref{ass:slow_obstacles}-\ref{ass:goal_distance} the proposed motion control scheme provide guaranteed convergence from the set $\mathcal{C}_s^{\bar{\rho}} = \mathcal{C}^{\bar{\rho}}(t_0)$ given by \eqref{eq:valid_rho} as stated by the following proposition.  
\begin{proposition}[Convergence to goal by successful DSW generation]
\label{theorem:convergence}
The trajectory for a robot with dynamics \eqref{eq:robot_model} following the motion control scheme in Algorithm \ref{alg:setpoint_control} converges to $p^g$ from any position $p(t_0)\in\mathcal{C}^{\bar{\rho}}_s$ if $E^{\star}(t_0)$ is a DSW and Assumptions \ref{ass:slow_obstacles}-\ref{ass:goal_distance} hold.
\end{proposition}

\textit{Proof:} See Appendix \ref{a:proof_convergence}.

The convergence in Proposition \ref{theorem:convergence} is dependent on a successful environment modification at time $t_0$. However, Theorem \ref{theorem:dsw_success} can be used to declare an a priori sufficient condition for convergence as stated below.
\begin{theorem}[Convergence to goal]
\label{theorem:convergence_apriori}
The trajectory for a robot with dynamics \eqref{eq:robot_model} following the motion control scheme in Algorithm \ref{alg:setpoint_control} converges to $p^g$ from any position $p(t_0)\in\mathcal{C}^{\bar{\rho}}_s$ if $E^{\bar{\rho}}_s$ is DSW equivalent and Assumptions \ref{ass:slow_obstacles}-\ref{ass:goal_distance} hold.
\end{theorem}
\begin{proof}
Since $p(t_0)\in \mathcal{C}^{\bar{\rho}}(t_0)$, Algorithm \ref{alg:obstacle_transformation} yields $\rho=\bar{\rho}$. Then $E^{\rho}(t_0)=E^{\bar{\rho}}_s$ is DSW equivalent. By design we have $r^0(t_0)\in\mathcal{F}^{\rho}(t_0)$ and $r^g(t_0)\in\mathcal{F}^{\rho}(t_0)$. From Theorem \ref{theorem:dsw_success} it can then be concluded that $E^{\star}(t_0)$ is a DSW and convergence to $p^g$ follows from Proposition \ref{theorem:convergence}.
\end{proof}

The positions from where convergence is not guaranteed, $\mathcal{F}\setminus\mathcal{C}^{\bar{\rho}}_s$, appear in the neighborhood of obstacle intersections, narrow passages, and in the neighborhood of concave obstacle vertices as seen in Figs. \ref{fig:rho_selection} and \ref{fig:static_setpoint}.

\section{Path-following with Obstacle Avoidance}
\label{sec:path_following_controller}
In this section, a path-following controller with obstacle avoidance is proposed, addressing Problem \ref{problem:path_following}. The control scheme closely follows the approach for setpoint stabilization presented in Section \ref{sec:setpoint_stabilization} with two additional components. Firstly, an alternative mapping for the RHRP is defined to follow the reference path $\Gamma$. We will refer to this as the \textit{nominal RHRP}. Secondly, an update policy is defined for the path parameter, $\theta$, such that $\gamma(\theta)$ moves in forward direction along $\Gamma$. To accomplish collision avoidance within the path-following scheme, the controller operate in two modes: \textit{nominal path mode} and \textit{collision avoidance mode}. The nominal path mode is active when the nominal RHRP has a clearance $\bar{\rho}$ to the environment, whereas the collision avoidance mode is active otherwise.
The complete path-following scheme is given in Algorithm \ref{alg:path_following_control}, and the details of the added components are described below.

\LinesNumbered
\begin{algorithm}[ht!]
    \caption{Path-following control scheme}
    \label{alg:path_following_control}
    \SetKwRepeat{Do}{do}{while}
    \SetKwInOut{Input}{Input}
    \SetKwInOut{Output}{Output}
    \SetKwInOut{Parameters}{Parameters}
    \SetKwInOut{Init}{Init}
    \Parameters{$\bar{\rho}\in\mathbb{R}^+$, $\alpha\in(0,1)$, $\Delta t\in\mathbb{R}^+$, $N\in\mathbb{N}^+$, $\lambda\in(0,1)$, $c_w\in\mathbb{R}^+$, $c_e\in\mathbb{R}^+$, $\kappa\in\mathcal{K}$, $\Gamma$, $L^{\textnormal{nom}}\in\mathbb{R}^+$}
    \Input{$E$, $x$}
    \Output{$\mu(\cdot)$}
    \Init{$r^+ \gets p$, $\theta \gets 0$}
    $\rho \gets \bar{\rho}$\;
    Compute $\mathcal{P}^{\textnormal{nom}}$ according to \eqref{eq:nominal_rhrp} and \eqref{eq:nominal_rhrp_mapping}\;\label{l:nominal_rhrp}
    \eIf{$\mathcal{P}^{\textnormal{nom}}\subset \mathcal{F}^{\rho}$\label{l:nominal_rhrp_valid}}{
        $\mathcal{P}\gets\mathcal{P}^{\textnormal{nom}}$\;
    }{
        $\theta \gets$ first $s \in [\theta^{\textnormal{nom}}(L^{\textnormal{nom}}), \theta^g]$ s.t. $\gamma(s)\in\mathcal{F}^{\rho}$\;\label{l:theta_update_ds}
        $p^g\gets \gamma(\theta)$\;
        Compute $E^{\star},r^0,r^g,\rho$ using Algorithm \ref{alg:obstacle_transformation}\;
        Compute $\mathcal{P}$ according to \eqref{eq:receding_reference_path}-\eqref{eq:receding_reference_dynamics}\;\label{l:ds_rhrp}
    }
    Run solver for \eqref{eq:mpc}\;
    \If{\textnormal{\textbf{MPC MODE and}} $r^+ \in \mathbb{B}[\gamma(\theta), \rho]$\label{l:theta_update3}}{
        $\theta \gets \theta + w_0^*\Delta t$\;\label{l:theta_update}
    }
    Update $\mu$ and $r^+$ according to \eqref{eq:scl}-\eqref{eq:r_plus}\;
\end{algorithm}

\subsubsection*{Nominal RHRP (line \ref{l:nominal_rhrp})}
The nominal RHRP is constructed to move from the initial reference candidate, $r^+$, towards the position specified by the current path parameter, $\theta$. Specifically, the reference initially travels along the line $l_{r^+}=l[r^+,\gamma(\theta)]$ after which it follows $\Gamma$ from $\gamma(\theta)$ to the end. The nominal RHRP is formally defined as
\begin{equation}
\label{eq:nominal_rhrp}
    \mathcal{P}^{\textnormal{nom}} = \left\{r\in \mathbb{R}^2 : s \in [0, L^{\textnormal{nom}}] \rightarrow r^{\textnormal{nom}}(s) \right\},
\end{equation}
where $L^{\textnormal{nom}}\in[L,\infty)$ is a parameter specifying the length of the nominal RHRP. The reference mapping is given by
\begin{equation}
\label{eq:nominal_rhrp_mapping}
    r^{\textnormal{nom}}(s) = \begin{cases}
    \frac{1-s}{\lVert l_{r^+} \rVert_2}r^+ + \frac{s}{\lVert l_{r^+} \rVert_2}\gamma(\theta), &s < \lVert l_{r^+} \rVert_2,\\
    \gamma\left(\theta^{\textnormal{nom}}(s)\right), &s \geq \lVert l_{r^+} \rVert_2,
\end{cases}
\end{equation}
where $\theta^{\textnormal{nom}}(s)$ shifts the value $s$ according to
\begin{equation}
    \theta^{\textnormal{nom}}(s) = \textnormal{sat}(\theta + s - \lVert l_{r^+} \rVert_2,\theta,\theta^g)
\end{equation}
with $\textnormal{sat}(\cdot,lb,ub)=\max(lb,\min(ub,\cdot))$ being the saturation function.

\subsubsection*{RHRP usage (lines \ref{l:nominal_rhrp_valid}-\ref{l:ds_rhrp})}
The nominal path is used when $\mathcal{P}^{\textnormal{nom}}\subset\mathcal{F}^{\bar{\rho}}$, i.e., the nominal RHRP has a clearance $\bar{\rho}$ to the environment. In this case, the assignment $\mathcal{P}=\mathcal{P}^{\textnormal{nom}}$ ensures a collision-free movement along $\Gamma$ when applying the control law according to \eqref{eq:scl}. 
If $\mathcal{P}^{\textnormal{nom}}\not\subset\mathcal{F}^{\bar{\rho}}$, the RHRP is instead computed using the environment modification (Section \ref{sec:workspace_modification}) and the DS-based approach (Section \ref{sec:rhrp}) to ensure collision avoidance. The setpoint for the RHRP genaration is the first collision-free position (w.r.t. $\mathcal{F}^{\bar{\rho}}$) along $\Gamma$ after the endpoint of the nominal RHRP. This enables the generation of a path which circumvents obstacles that obstructs the nominal RHRP. 

\subsubsection*{Path parameter update (line \ref{l:theta_update})}
In accordance with the logic for the update of $r^+$ \eqref{eq:r_plus}, the path parameter is incremented by an amount equivalent to the one-step reference increment of the MPC solution, i.e., $w^*_0\Delta t$.
This is executed whenever the initial reference candidate lies within a $\rho$-neighborhood of the position specified by the current path parameter, i.e., when $r^+ \in \mathbb{B}[\gamma(\theta), \rho]$. The extra condition is included to prevent the path parameter from diverging from the robot's position after any path parameter shift made in collision avoidance mode (line \ref{l:theta_update_ds}).


\section{Results}
\label{sec:results}
To illustrate the performance of the control schemes, various simulation scenarios are carried out. A unicycle robot described by \eqref{eq:unicycle} and input constraints $\mathcal{U}=[-0.1, 1] \times [-1, 1]$ is considered in all cases. The Runge-Kutta method (RK4) is applied for integration of the system evolution which is updated at a frequency of $100$ Hz. Function approximation of the RHRP using a sixth-degree polynomial and RHRP buffering are applied as described in \cite{dahlin_karayiannidis_23_2}. The regularization term is defined to smoothen the trajectory as $J(z)=\sum_{i=0}^{N-1}(\bar{u}_i-u^d)^TR(\bar{u}_i-u^d)+(\bar{u}_i-\bar{u}_{i-1})^T_iR_{\Delta}(\bar{u}_i-\bar{u}_{i-1})$, with $u^d = [w^{\max}\ 0]^T$ being the desired control input and $\bar{u}_{-1}$ being the previously applied control input. The control sampling period is $\Delta t=0.2$ and the state integration in the MPC is performed using RK4. The SBC is defined as in \eqref{eq:unicycle_kappa} with $k_1=0.15, k_2=0.3$ ensuring $\kappa(x,r)\in\mathcal{U},\ \forall x\in\mathcal{X}, \forall r\in\mathbb{R}^2$. 
All numerical values for the control parameters are stated in Table \ref{tab:params}. 
To compare the result of the proposed path-following controller, we consider the conventional approach for path-following MPC with obstacle avoidance - a straightforward addition of ``no-go-zones'' for the robot trajectory \cite{arbo_etal_17,zube_15}. That is, using the MPC problem formulation \eqref{eq:mpc} where \eqref{eq:mpc_error} and \eqref{eq:mpc_w0_con} are replaced by the constraint $p(\tau)\not\in\mathcal{O}$ (this can efficiently be encoded as presented in \cite{zhang_etal_21}). We will refer to this method as the \textit{standard MPPFC}. Additionallly, evaluation of the xMPPFC \cite{sanchez_etal_21} is included, which incorporates an auxiliary trajectory. The tuning parameters are the same as for the proposed approach for the standard MPPFC, and are defined as in \cite{sanchez_etal_21} for the xMPPFC. For both standard MPPFC and xMPPFC the prediction horizon is twice as long ($N=12$) compared to the proposed method since converging behavior for these approaches depends strongly on the horizon length. Moreover, full knowledge of future obstacle poses within the horizon is assumed, whereas the proposed method only utilizes the current obstacle poses.

\begin{table}[!t]
\renewcommand{\arraystretch}{1.3}
\caption{Control parameters}
\label{tab:params}
\centering
\begin{tabular}{|c|c|c|c|c|c|c|c|c|c|}
\hline
\bfseries $\bar{\rho}$ & $\alpha$ & $N$ & $\lambda$ & $c_w$ & $c_e$ & R & $R_{\Delta}$ & $L^{\textnormal{nom}}$\\
\hline
0.2 & 0.9 & 6 & 0.5 & 1 & 1 & $
\setstacktabbedgap{2pt}
\parenMatrixstack{
0.1 & 0 \cr
0 & 0.1}$ 
& $
\setstacktabbedgap{2pt}
\parenMatrixstack{
0.1 & 0 \cr
0 & 0}$ & 2\\
\hline
\end{tabular}
\end{table}

\subsection{Setpoint Stabilization}
To illustrate the convergence properties derived in Section \ref{sec:setpoint_stabilization_scheme}, two static scenes as shown in Fig. \ref{fig:static_setpoint} are considered. The robot is initialized at different positions $p(t_0)\in \mathcal{F}$ with horizontal orientation, $\psi(t_0)=0$, for all cases. In Fig. \ref{fig:static_setpoint_dsw}, the environment form a DSW equivalent $\mathcal{F}^{\bar{\rho}}_s$ and convergence can be concluded a priori from any position $p(t_0)\in\mathcal{C}^{\bar{\rho}}_s$ by Theorem \ref{theorem:convergence_apriori} which is confirmed by the simulation results. The environment in Fig. \ref{fig:static_setpoint_non_dsw} is not DSW equivalent. However, $E^{\star}(t_0)$ is a DSW for all given initial positions and convergence to $p^g$ from any $p(t_0)\in\mathcal{C}^{\bar{\rho}}_s$ follows from Proposition \ref{theorem:convergence} which is also confirmed by the simulations. The robot is also initialized at one position in the set $\mathcal{F}\setminus\mathcal{C}^{\bar{\rho}}_s$ from where Proposition \ref{theorem:convergence} provides no convergence guarantee (lower left). Nonetheless, the robot converges to $p^g$, indicating stronger convergence than is theoretically proved. Note that the shape of the obstacles $\mathcal{O}^{\star}$ depends on the robot position, and $E^{\star}$ is thus different for each case as illustrated in Fig. \ref{fig:starworlds}. 

\begin{figure}[ht!]
\begin{subfigure}[t]{0.48\linewidth}
        \includegraphics[width=\linewidth,trim={3.5cm 1.3cm 2.8cm 1.5cm},clip]{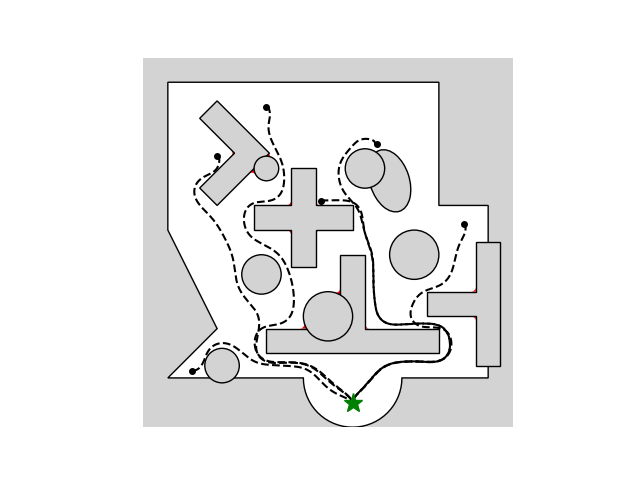}
     \caption{$E^{\bar{\rho}}_s$ is DSW equivalent and convergence to $p^g$ from initial positions $p(t_0)\in\mathcal{C}^{\bar{\rho}}_s$ can be concluded by Theorem \ref{theorem:convergence_apriori}.}
     \label{fig:static_setpoint_dsw}
    \end{subfigure}
    \hfill
    \begin{subfigure}[t]{0.481\linewidth}
        \includegraphics[width=\linewidth,trim={3.5cm 1.3cm 2.8cm 1.5cm},clip]{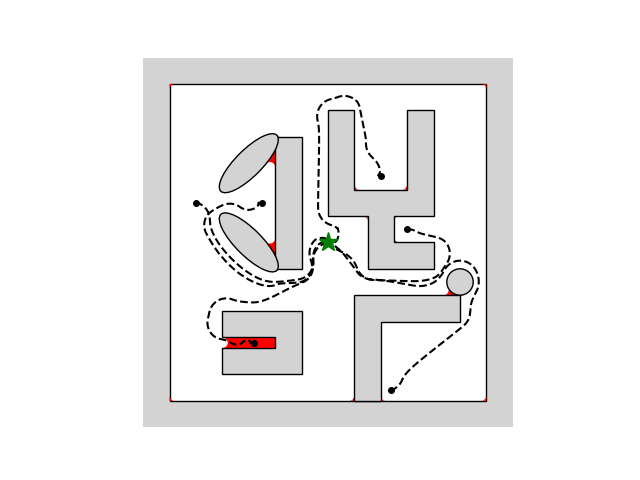}
         \caption{$E^{\star}(t_0)$ is a DSW for all given initial positions and convergence to $p^g$ from any $p(t_0)\in\mathcal{C}^{\bar{\rho}}_s$ can be concluded by Proposition \ref{theorem:convergence}.}
         \label{fig:static_setpoint_non_dsw}
    \end{subfigure}
    \caption{A static set of obstacles $\mathcal{O}$ (grey), and the set $\mathcal{F}\setminus \mathcal{C}^{\bar{\rho}}_s$ (red) from where convergence cannot be stated by Proposition \ref{theorem:convergence}. Travelled path (dashed black lines) to a goal position $p^g$ (green star) is shown from different initial positions $p(t_0)$ (black dots), all with horizontal initial orientation $\psi(t_0)=0$.}
    \label{fig:static_setpoint}
\end{figure}

\begin{figure}[ht!]
    \centering
    \begin{subfigure}[t]{0.45\linewidth}
        \includegraphics[width=\linewidth,trim={.5cm .5cm .5cm .5cm},clip]{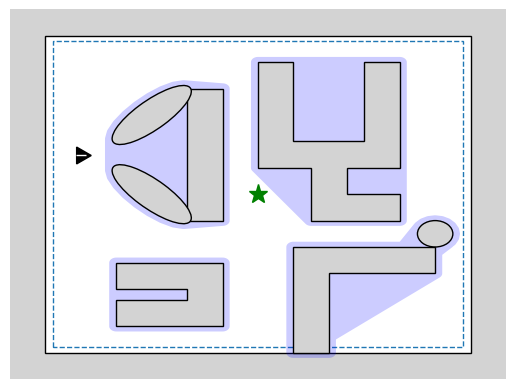}
    \end{subfigure}
    \hfill
    \begin{subfigure}[t]{0.45\linewidth}
        \includegraphics[width=\linewidth,trim={.5cm .5cm .5cm .5cm},clip]{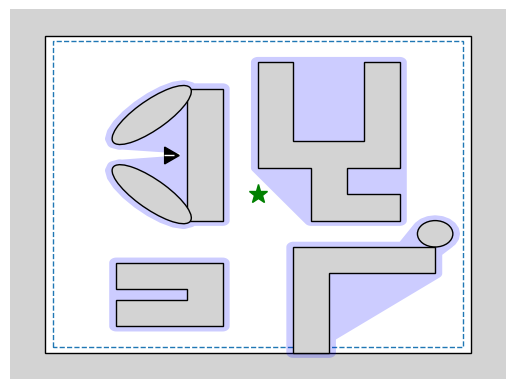}
    \end{subfigure}
    
    \vspace{0.2cm}
    
    \begin{subfigure}[t]{0.45\linewidth}
        \includegraphics[width=\linewidth,trim={.5cm .5cm .5cm .5cm},clip]{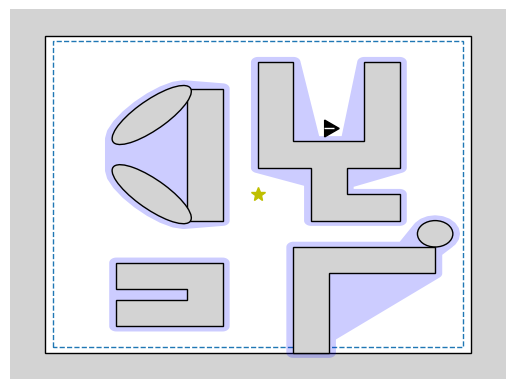}
    \end{subfigure}
    \hfill
    \begin{subfigure}[t]{0.45\linewidth}
        \begin{tikzpicture}[      
                every node/.style={anchor=south west,inner sep=0pt},
                x=1mm, y=1mm,
              ]   
             \node (fig1) at (0,0)
               {\includegraphics[width=\linewidth,trim={.5cm .5cm .5cm .5cm},clip]{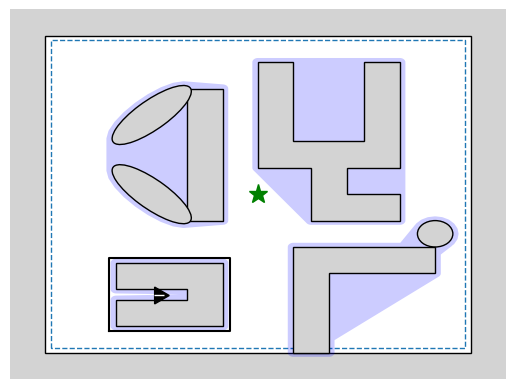}};
             \node (fig2) at (16,12)
               {\includegraphics[width=0.7\linewidth,trim={.2cm .2cm .2cm .2cm},clip]{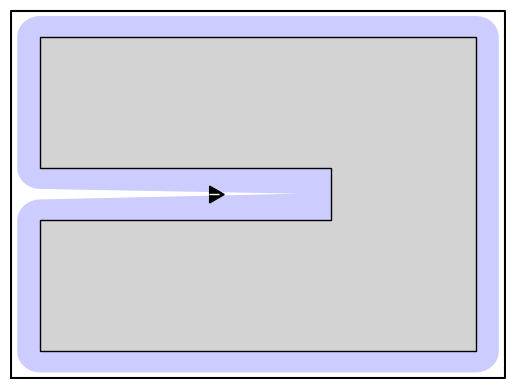}};
            \draw[densely dotted, line width=.8pt] (7.2, 9.6) -- (fig2.north west);
            \draw[densely dotted, line width=.8pt] (17.6, 3) -- (fig2.south east);
        \end{tikzpicture}
    \end{subfigure}
    \caption{The starshaped obstacles $\mathcal{O}^{\star}$ (blue) are formed to obtain a starshaped enclosing of the obstacles $\mathcal{O}$ (grey) while remaining the robot position (black triangle) an exterior point.}
    \label{fig:starworlds}
\end{figure}

As suggested in \cite{huber_etal_22}, also non-starshaped workspaces can be treated by dividing the workspace into several ordered subregions with corresponding local dynamics generated by a high-level planner. In Fig. \ref{fig:several_boundaries}, a concave workspace is divided into four intersecting rectangles, each assigned with a goal point that guides the robot towards the next subregion. The subregions are activated as current workspace in a consecutive manner when the robot enters the region interior. The environment contains four moving circular obstacles and one static polygon. At time $12.5$s, the robot enters on the right side of the polygon obstacle. At time $13.6$s, the closest circular obstacle has moved towards the polygon such that the gap between them has closed. At this point, the RHRP drastically changes to circumvent the polygon on the left side. Due to the limited rotational velocity and reverse speed of the robot, the enforced initial forward reference motion \eqref{eq:mpc_w0_con} is conflicting with the tracking error constraint \eqref{eq:mpc_error} and the MPC problem is infeasible. During the time period $t\in [13.6,15.2]$, the SBC is applied and the robot realigns with the RHRP, enabling feasibility of \eqref{eq:mpc} at $t=15.2$s and afterwards.

\def\figscale{0.75}
\begin{figure}[H]
    \begin{subfigure}[t]{0.48\linewidth}
        \centering
            \includegraphics[width=\figscale\linewidth,trim={1cm .5cm 1cm 0.5cm},clip]{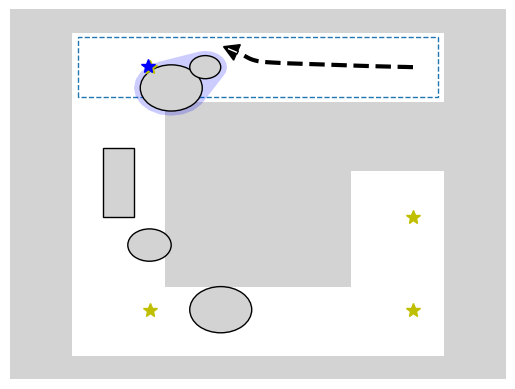}
         \caption{$t=5s$}
         \label{fig:several_boundaries_1}
    \end{subfigure}
    \hfill
    \begin{subfigure}[t]{0.48\linewidth}
        \centering
            \includegraphics[width=\figscale\linewidth,trim={1cm .5cm 1cm 0.5cm},clip]{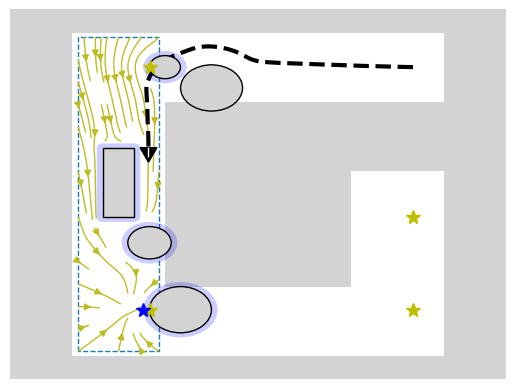}
         \caption{$t=12.5s$. The RHRP guides the robot to the right of the polygon obstacle as suggested by the vector field of \eqref{eq:receding_reference_dynamics} (yellow arrows).}
         \label{fig:several_boundaries_2}
    \end{subfigure}
    \begin{subfigure}[t]{0.48\linewidth}
        \centering
            \includegraphics[width=\figscale\linewidth,trim={.5cm .5cm .5cm 0.5cm},clip]{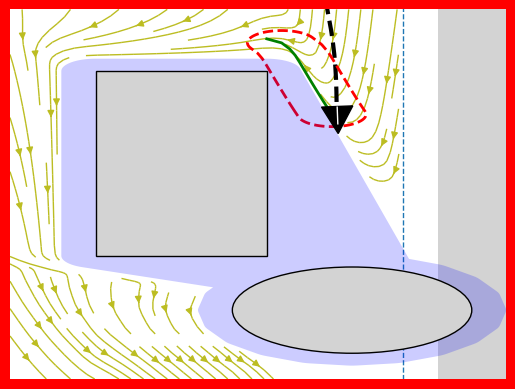}
         \caption{$t=13.6s$. The vector field of \eqref{eq:receding_reference_dynamics} (yellow arrows), and thus the RHRP, has drastically changed, so that the MPC problem is infeasible and the controller enters \textnormal{\textbf{SBC MODE}}.}
         \label{fig:several_boundaries_3}
    \end{subfigure}
    \hfill
    \begin{subfigure}[t]{0.48\linewidth}
        \centering
            \includegraphics[width=\figscale\linewidth,trim={.5cm .5cm .5cm 0.5cm},clip]{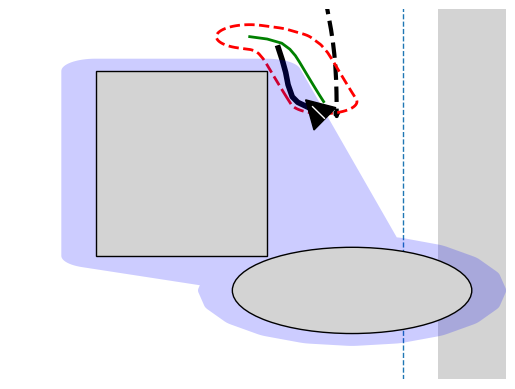}
         \caption{$t=15.2s$. The robot has realigned with the RHRP, making the MPC problem feasible again (with solution path shown as solid black line), and the controller enters \textnormal{\textbf{MPC MODE}}.}
         \label{fig:several_boundaries_4}
    \end{subfigure}
    \begin{subfigure}[t]{0.48\linewidth}
        \centering
            \includegraphics[width=\figscale\linewidth,trim={1cm .5cm 1cm 0.5cm},clip]{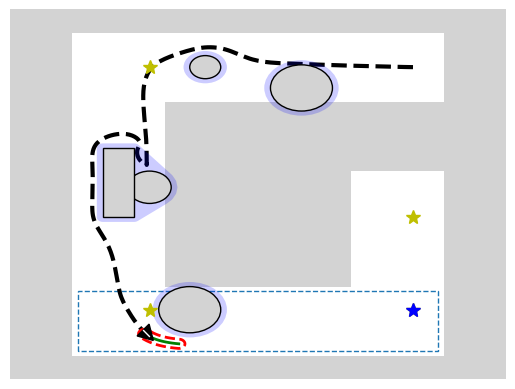}
         \caption{$t=27s$}
         \label{fig:several_boundaries_5}
    \end{subfigure}
    \hfill
    \begin{subfigure}[t]{0.48\linewidth}
        \centering
            \includegraphics[width=\figscale\linewidth,trim={1cm .5cm 1cm 0.5cm},clip]{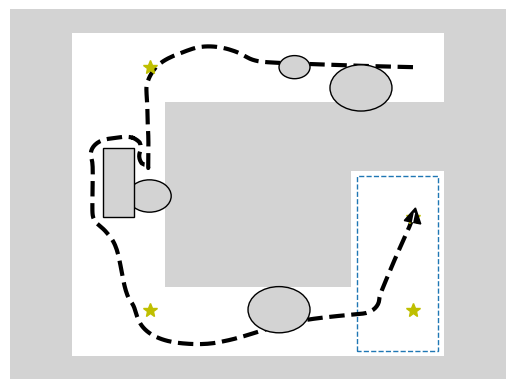}
         \caption{$t=40s$}
         \label{fig:several_boundaries_6}
    \end{subfigure}
    \caption{The robot navigates in a non-starshaped workspace divided into several subregions with corresponding goal points (yellow stars) to the final goal. The active clearance workspace $\mathcal{W}^{\rho}(t)$ (blue dashed line), starshaped obstacles $\mathcal{O}^{\star}(t)$ (blue), and active goal point (blue star) depends on the robot position.}
    \label{fig:several_boundaries}
\end{figure}

A scenario where the robot navigates in a crowded dynamic environment where each obstacle is a circle of radius 0.5 m is also demonstrated in Fig. \ref{fig:crowd_setpoint}. The workspace is considered as a disk of radius 4 m moving with the robot, $\mathcal{W}(t)=\mathbb{B}[p(t), 4]$. Any obstacle touching this region is identified by the robot and included in the obstacle set, $\mathcal{O}(t)$. The robot is initialized at five different positions with vertical orientation and all obstacles move straight downwards in the scene with speeds between $0.1$ and $0.3$ m/s. In all cases, the robot avoids the obstacles and converges to $p^g$. 

\def\figscale{0.31}
\def\figheight{4.5cm}
\begin{figure}[ht!]
    \begin{subfigure}[t]{\figscale\linewidth}
        \centering
        \includegraphics[height=\figheight,trim={2cm 0.4cm 2.2cm .8cm},clip]{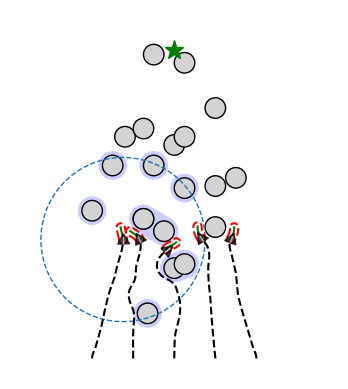}
     \caption{$t=6$s.}
     \label{fig:crowd_1}
    \end{subfigure}
    \hfill
    \begin{subfigure}[t]{\figscale\linewidth}
        \centering
        \includegraphics[height=\figheight,trim={2cm 0.4cm 2.2cm .8cm},clip]{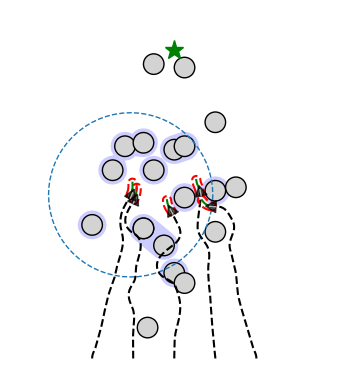}
         \caption{$t=8.3$s.}
         \label{fig:crowd_2}
    \end{subfigure}
    \hfill
    \begin{subfigure}[t]{\figscale\linewidth}
        \centering
        \includegraphics[height=\figheight,trim={2cm 0.4cm 2.2cm .8cm},clip]{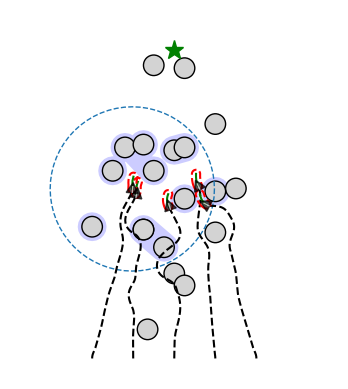}
         \caption{$t=8.6$s.}
         \label{fig:crowd_3}
    \end{subfigure}
    \begin{subfigure}[t]{\figscale\linewidth}
        \centering
        \includegraphics[height=\figheight,trim={2cm 0.4cm 2.2cm .8cm},clip]{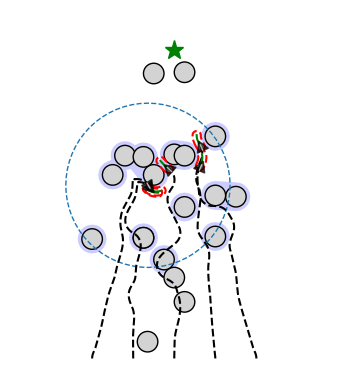}
         \caption{$t=10.6$s.}
         \label{fig:crowd_4}
    \end{subfigure}
    \hfill
    \begin{subfigure}[t]{\figscale\linewidth}
        \centering
        \includegraphics[height=\figheight,trim={2cm 0.4cm 2.2cm .8cm},clip]{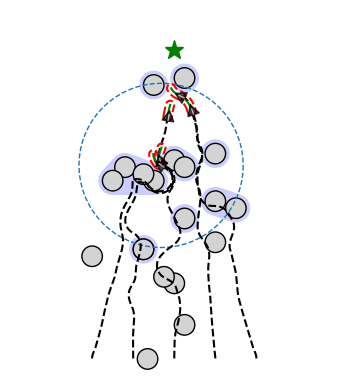}
         \caption{$t=13.4$s.}
         \label{fig:crowd_5}
    \end{subfigure}
    \hfill
    \begin{subfigure}[t]{\figscale\linewidth}
        \centering
        \includegraphics[height=\figheight,trim={2cm 0.4cm 2.2cm .8cm},clip]{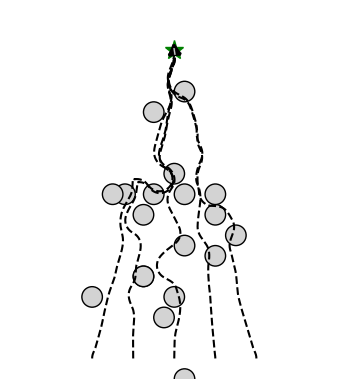}
         \caption{$t=20$s.}
         \label{fig:crowd_6}
    \end{subfigure}
    \caption{A set of moving obstacles $\mathcal{O}$ (grey) and five robot positions (black triangles) with travelled paths (dashed black lines) and RHRP (green line) along with boundary of accompanying tunnel region $\mathcal{P}\oplus \mathbb{B}[0,\rho]$ (dashed red line). The workspace boundary $\partial \mathcal{W}$ (dashed blue line) and starshaped obstacles $\mathcal{O}^{\star}$ (blue) are depicted for the leftmost robot position.}
    \label{fig:crowd_setpoint}
\end{figure}

\subsection{Path-following}
In Fig. \ref{fig:infinity_pf_2}, an $\infty$-shaped reference path is used, given by $\gamma(\theta) = \big[6\cos\left(\frac{2\pi}{\theta^g}\theta\right),\allowbreak 3\sin\left(\frac{4\pi}{\theta^g}\theta\right)\big]$ with $\theta^g=36.5$. The path is obstructed by a circular obstacle, three intersecting circular obstacles and a starshaped polygon. Each time the nominal RHRP penetrates the clearance obstacles, $\mathcal{O}^{\rho}$, the RHRP is generated using the setpoint stabilization approach. The setpoint is given as the first collision-free position after the nominal RHRP along $\Gamma$ as illustrated in Fig. \ref{fig:infinity_pf_1}. In comparison, the standard MPPFC successfully circumvents the first obstacle, but as the path gets blocked by the three intersecting obstacles, the robot comes to a full stop. This can be explained by the extensive path deviation that is needed to circumvent the obstacles which is not motivated by the cost function within the horizon. The xMPPFC successfully circumvent both the first obstacle and the three intersecting obstacles, but comes to full stop as it reaches the polygon obstacle.

\begin{figure}[ht!]
    \centering
        \includegraphics[width=0.6\linewidth]{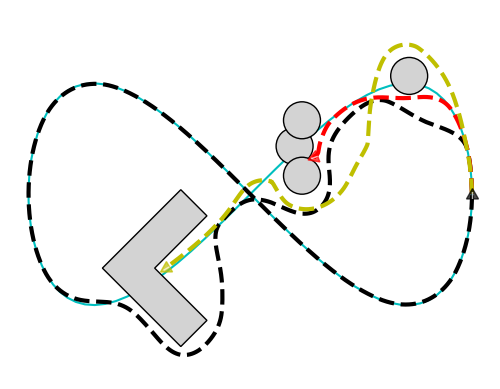}
      \caption{The robot successfully avoids the obstacles (grey) and returns to follow $\Gamma$ (cyan line) when using the proposed path-following controller (black dashed line). The robot comes to full stop by the three intersecting obstacles when using the standard MPPFC (red dashed line), and by the starshaped polygon when using the xMPPFC (yellow dashed line).}
     \label{fig:infinity_pf_2}
\end{figure}
\begin{figure}[ht!]
    \centering
    \includegraphics[width=0.6\linewidth]{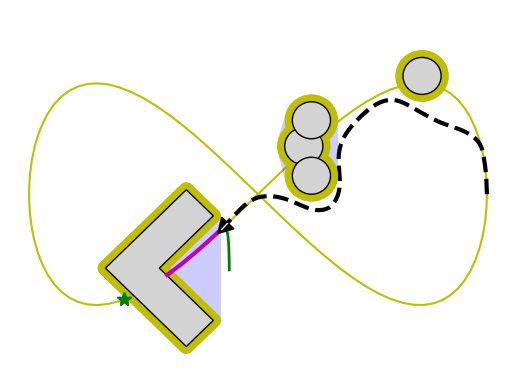}
     \caption{The nominal RHRP (magenta line) intersects the clearance obstacles $\mathcal{O}^{\rho}$ (yellow) and the RHRP (green line) is generated using the setpoint stabilization approach with goal position given as next collision-free position after the nominal RHRP (green star).}
     \label{fig:infinity_pf_1}
\end{figure}

Instead of providing local goal points for each subregion as in the example of Fig. \ref{fig:several_boundaries}, an alternative approach involves the high-level planner supplying a global path based on the static environment. This can be used by a path-following controller, exemplified in Fig. \ref{fig:several_boundaries_pf}.
As observed, the proposed path-following controller yields a comparable outcome to the one shown in Figure \ref{fig:several_boundaries}, with the distinction that the introduced path is now closely tracked whenever feasible. Similar to the example depicted in Fig. \ref{fig:infinity_pf_2}, both the standard MPPFC and xMPPFC can circumvent the initial obstacles to continue along the path, but as a major deviation from the path is needed, the robot comes to a full 

\begin{figure}[ht!]
    \centering
    \includegraphics[width=0.52\linewidth]{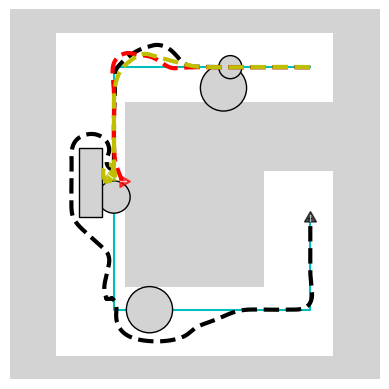}
    \caption{The path-following controller with a given path (cyan line) is used in the same dynamic scenario as in Fig \ref{fig:several_boundaries}. The robot comes to full stop both when using standard MPPFC (red dashed line) and xMPPFC (yellow dashed line), whereas convergence to the final position is achieved when using the proposed control scheme (black dashed line).}
    \label{fig:several_boundaries_pf}
\end{figure}

\section{Conclusion}
This article proposed a motion control scheme for robots operating in a workspace containing a collection of dynamic, possibly intersecting, obstacles. Both the setpoint stabilization and path-following problem was treated.
The method combines environment modification into a scene of disjoint obstacles with a closed form dynamical system formulation to generate a receding horizon path. A novel MPC formulation was proposed to enforce forward motion along the path within an obstacle-clearance zone, which in combination with a stabilizing backup controller allowed for formal derivations of collision avoidance and convergence properties.

As the method relies on conservative treatment of the obstacle regions, an inherent drawback is the possible gap closing in narrow passages. In its current form, only current obstacle positions is considered, and an extension to enable incorporation of predicted obstacle poses would be beneficial to prevent unnecessary maneuvers.

\section*{Appendix}
\label{paperF:sec:appendix}
\addcontentsline{toc}{section}{Appendix}
\textsl{A. Kernel point selection for ModEnv\texorpdfstring{$^{\star}$}{Star}}

Algorithm \ref{alg:kernel_point_selection} computes the kernel points, $K$, for generating a starshaped obstacle using starshaped hull with specified kernel. The input are the cluster of obstacles to be enclosed by the generated starshape, $cl$, the points to exclude in the obstacle region, $\bar{X}$, and the kernel centroid for $cl$ at previous iteration (if applicable), $k_c^{prev}$. The kernel point selection is extended from Algorithm 3 in \cite{dahlin_karayiannidis_23_1} in two ways: 1) in scenarios with confined workspace, the kernel selection is restricted to the workspace exterior for any cluster which is not fully contained in the workspace, if possible (line \ref{l:kernel_workspace1}-\ref{l:kernel_workspace2}), 2) the kernel selection is restricted to the intersecting kernel region of the clustered obstacles, if possible (lines \ref{l:intersecting_kernel_condition}-\ref{l:intersecting_kernel}). For more details on the full procedure, the reader is referred to \cite{dahlin_karayiannidis_23_1}.

\LinesNumbered
\begin{algorithm}
    \caption{Kernel point selection for ModEnv$^{\star}$} 
    \label{alg:kernel_point_selection}
    Line 1 in Algorithm 3 of \cite{dahlin_karayiannidis_23_1}\;
    \If{\textnormal{\textbf{not}} $cl_{\cup}\subset\mathcal{W}$ \textnormal{\textbf{and}} $S \setminus \mathcal{W}\neq\emptyset$\label{l:kernel_workspace1}}{
        $S\gets S \setminus \mathcal{W}$\label{l:kernel_workspace2}\;
    }
    \uIf{$\forall \mathcal{O}^i\in cl$ \textnormal{starshaped \textbf{and}} $S \cap \textnormal{ker}_{\cap}(cl) \neq \emptyset$\label{l:intersecting_kernel_condition}}{
        $S \gets S \cap \textnormal{ker}_{\cap}(cl)$\;\label{l:intersecting_kernel}
    }
    \uElseIf{$S \cap cl_{\cup} \neq \emptyset$}{
        $S \gets S \cap cl_{\cup}$\;
    }
    Lines 4-13 in Algorithm 3 of \cite{dahlin_karayiannidis_23_1}\;
\end{algorithm}

\noindent\textsl{B. Proof of Theorem \ref{theorem:dsw_success}}

\label{a:dsw_success}
Given the cluster set, $Cl$, at the start of an iteration of Algorithm 2 of \cite{dahlin_karayiannidis_23_1}, it follows from Property 4b and 4d of \cite{dahlin_karayiannidis_23_1} that for the generated starshaped obstacles set at that iteration, $\mathcal{O}^{\star}$, it holds for any $i\in[1,..,\lvert Cl \rvert]$ that $K^i\subset\textnormal{ker}_{\cap}(cl^i)\Rightarrow \mathcal{O}^{\star,i}=SH_{\textnormal{ker}K^i}(cl^i)=cl_{\cup}^i$. 
The kernel selection Algorithm 1 computes $K^i\subset S^i$ with $S^i = \textnormal{ker}_{\cap}(cl^i)$ if $cl_{\cup}\subset\mathcal{W}$ and $S^i=\textnormal{ker}_{\cap}(cl^i)\setminus \mathcal{W}$ if $cl_{\cup}\not\subset\mathcal{W}$, given that $S^i$ is nonempty for these selections.

At first iteration in Algorithm 2 of \cite{dahlin_karayiannidis_23_1} we have $Cl=\mathcal{O}$ and $\textnormal{ker}_{\cap}(cl^i)=\textnormal{ker}(\mathcal{O}^i)$. Since the environment is DSW equivalent, the set $S^i$ specified above is nonempty. Specifically, $\textnormal{ker}(\mathcal{O}^i)$ is nonempty since $\mathcal{O}^i$ is starshaped and $\textnormal{ker}(\mathcal{O}^i)\setminus\mathcal{W}$ is nonempty for any obstacle $\mathcal{O}^i\not\subset\mathcal{W}$ due to condition \eqref{eq:dsw_workspace_kernel}. Thus, $K^i\subset\textnormal{ker}_{\cap}(cl^i)$ and $\mathcal{O}^{\star,i}=\mathcal{O}^i,\ i\in[1,..,\lvert Cl\rvert]$. As a consequence, $Cl=Cl^{\star}$ after the assignment in line 13.
By construction, $Cl^{\star}$ consists of mutually disjoint connected subsets of $\mathcal{O}^{\star}$ and $Cl$ thus consists of mutually disjoint connected subsets of $\mathcal{O}$. If the environment is a DSW, it follows that $Cl=\mathcal{O}$ and the algorithm returns the DSW. Otherwise, $Cl$ satisfies \eqref{eq:dsw_kernel} and \eqref{eq:dsw_workspace_kernel} since the environment is DSW equivalent and the division of a set into  mutually disjoint connected subsets is unique. Hence, $S^i$ specified above is nonempty. Thus, $K^i\subset\textnormal{ker}_{\cap}(cl^i)$ and $\mathcal{O}^{\star,i}=cl_{\cup}^i,\ i\in[1,..,\lvert Cl\rvert]$. Hence, $\mathcal{O}^{\star}$ is a mutually disjoint subset of $\mathcal{O}$ with $\mathcal{O}^{\star}_{\cup}=\mathcal{O}_{\cup}$, and $\lvert \mathcal{O}^{\star} \rvert = \lvert Cl \rvert$ such that the algorithm terminates. Since all regions in $\mathcal{O}^{\star}$ are starshaped by construction, and where any clustered obstacle satisfies \eqref{eq:dsw_workspace_kernel} the environment  $\{\mathcal{W},\mathcal{O}^{\star}\}$ is a DSW.
\\

\noindent\textsl{C. Proof of Theorem \ref{theorem:collision_avoidance}}

\label{a:proof_collision_avoidance}
We have 
\begin{equation*}
\begin{split}
    &p(t)\in\mathcal{F}(t_k),\ \forall t\in(t_k, t_{k+1}] \\ \xRightarrow{Ass.\ref{ass:slow_obstacles}}\
    &p(t)\in\mathcal{F}(t),\ \forall t\in(t_k, t_{k+1})\ \&\ p(t_{k+1})\in\mathcal{F}(t_k)\\
    \xRightarrow{Ass.\ref{ass:not_aggressive_obstacles}}\
    &p(t)\in\mathcal{F}(t),\ \forall t\in(t_k, t_{k+1}].
\end{split}
\end{equation*}
Since $p(t_0)\in\mathcal{F}(t_0)$, it therefore suffices to show that $p(t)\in\mathcal{F}(t_k),\ \forall t\in(t_k, t_{k+1}]$ given $p(t_k)\in\mathcal{F}(t_k)$ at any sampling instance, $t_k$. 

Consider first the case where the controller is in \textbf{MPC MODE} at time $t_k$ and thus $u(t) = \bar{u}^*_0(t_k),\ \forall t\in [t_k,t_{k+1})$. Then $\lVert p(t+\tau)-r_k(w^*_{k,0}\tau)\rVert_2 \leq \rho(t_k), \forall \tau\in[0,\Delta t]$ from \eqref{eq:robot_model} and \eqref{eq:mpc_robot_dyn}-\eqref{eq:mpc_error}. Here $r_k(\cdot)\in\mathcal{P}(t_k)$ is the RHRP-mapping at time instance $t_k$ and $w^*_{k,0}$ is the initial path speed of the optimal solution $z^*(t_k)$. Thus, $p(t)\in\mathcal{P}^{\rho}(t_k)\subset\mathcal{F}(t_k),\ \forall t\in (t_k,t_{k+1}]$.
If the controller instead is in \textbf{SBC MODE}, the SBC is applied, $u(t)=\kappa(x(t),r^0(t_k)),\ \forall t\in [t_k,t_{k+1})$. Since $\lVert r^0(t_k)-p(t_k)\rVert_2\leq \rho(t_k)$ by definition of $r^0$, it follows from the definition of SBC that $\lVert r^0(t_k)-p(t)\rVert \leq \lVert r^0(t_k)-p(t_k)\rVert_2\leq \rho(t_k), \forall t\in (t_k,t_{k+1}]$. That is $p(t)\in\mathbb{B}[r^0(t_k),\rho(t_k)]\subset\mathcal{P}^{\rho}(t_k) \subset \mathcal{F}(t_k),\ \forall t\in(t_k, t_{k+1}]$.
\\

\noindent\textsl{D. Proof of Proposition \ref{theorem:convergence}}

\label{a:proof_convergence}
For ease of notation, let $\rho_k=\rho(t_k)$. The proof is given in five steps. In the first two steps, we show that the environment modification is static after $k=0$ in the sense that $\rho_k=\bar{\rho}, \forall k\geq0$ and $\mathcal{F}^{\star}_k=\mathcal{F}^{\star}_0, \forall k\geq0$. In step 3 we show that the initial reference point $r^0_k$ is following the parameterized regular curve $\hat{\Gamma}=\{\hat{\gamma} \in \mathbb{R}^2 : \hat{\theta} \in [0, \infty) \rightarrow \hat{\gamma}(\hat{\theta})\}$ given by 
\begin{equation}
\label{eq:gamma}
    \frac{d\hat{\gamma}(\hat{\theta})}{d\hat{\theta}} = \bar{\nu}(\hat{\gamma}(\hat{\theta}),p^g, \mathcal{O}^{\star}_0),\quad \hat{\gamma}(0) = r^+_0,
\end{equation}
which converges to $p^g$. Specifically, we show $r^0_k=\hat{\gamma}(\hat{\theta}_k),\forall k\geq 0$ given the virtual path coordinate $\hat{\theta}$ with $\hat{\theta}_0=0$ and dynamics
\begin{equation}
\label{eq:theta_dynamics}
    \hat{\theta}_{k+1} = \begin{cases}
    \hat{\theta}_k + w^*_{k,0}\Delta t,\quad & \textnormal{\textbf{MPC MODE}}\\
    \hat{\theta}_k,\quad & \textnormal{\textbf{SBC MODE}}.
    \end{cases}
\end{equation}
In step 4 we show that $r^0_k$ converge to $p^g$ in finite time, and in step 5 it is shown that this implies convergence of $p$ to $p^g$.

\textit{Step 1 $\left(\rho_k=\bar{\rho},\ \forall k\geq 0\right)$:}
Assume $\exists k\geq 0,\ \rho_k=\bar{\rho},\ \rho_{k+1}\neq\bar{\rho}$. From the proof of Theorem \ref{theorem:collision_avoidance}, with use of the fact $\mathcal{O}_{k+1}=\mathcal{O}_k$, we have $p_{k+1}\in \mathcal{P}^{\rho}_k\subset \mathcal{F}^{\bar{\rho}}_k\oplus \mathbb{B}[0,\bar{\rho}]=\mathcal{F}^{\bar{\rho}}_{k+1}\oplus \mathbb{B}[0,\bar{\rho}] = \mathcal{C}^{\bar{\rho}}_{k+1}$. Hence, Algorithm \ref{alg:obstacle_transformation} yields $\rho_{k+1}=\bar{\rho}$. This is a contradiction and we can conclude $\rho_k=\bar{\rho} \Rightarrow \rho_{k+1}=\bar{\rho},\ \forall k\geq 0$. Since $p_0\in \mathcal{C}^{\bar{\rho}}_0$ and therefore $\rho_0=\bar{\rho}$, we can conclude that $\rho_k=\bar{\rho},\ \forall k\geq 0$.

\textit{Step 2 ($\mathcal{F}^{\star}_k=\mathcal{F}^{\star}_0,\ \forall k\geq 0$):} Since $\rho_k=\bar{\rho},\ \forall k\geq 0$ it follows that $\mathcal{F}^{\star}_k\subset \mathcal{F}^{\rho}_{k+1},\ \forall k\geq 0$. Given $\mathcal{F}^{\star}_k$ is a DSW, it follows from Algorithm \ref{alg:obstacle_transformation} that $\mathcal{F}^{\star}_{k+1}=\mathcal{F}^{\star}_k$ if $r^0_{k+1}\in \mathcal{F}^{\star}_k$ and $r^g_{k+1}\in \mathcal{F}^{\star}_k$. 
From the proof of Theorem \ref{theorem:collision_avoidance} and from \eqref{eq:r_plus} we have $p_{k+1}\in \mathbb{B}[r_k(w^*_{k,0}\Delta t),\bar{\rho}]=\mathbb{B}[r_k^+,\bar{\rho}]$ when in \textbf{MPC MODE} and $p_{k+1}\in \mathbb{B}[r^0_k,\bar{\rho}]=\mathbb{B}[r_k^+,\bar{\rho}]$ when in \textbf{SBC MODE}. Since $r^+_k\in\mathcal{P}_k\subset\mathcal{F}^{\rho}_{k+1}$ by definition, it follows that $r^+_k \in \mathcal{P}^0_{k+1}$ and hence $r_{k+1}^0 = r_k^+$. Then, $r^0_{k+1}\in\mathcal{P}_k\subset\mathcal{F}^{\star}_k$.
The reference goal is given by $r^g_{k+1} = \argmin_{r^g\in \mathcal{F}_{k+1}^{\rho}} \lVert r^g - p^g \lVert_2 = \argmin_{r^g\in \mathcal{F}_k^{\rho}} \lVert r^g - p^g \lVert_2 = r^g_k \in\mathcal{F}^{\star}_k$. Thus, $\mathcal{F}^{\star}_k$ being a DSW implies $\mathcal{F}^{\star}_{k+1}=\mathcal{F}^{\star}_k,\ \forall k\geq 0$. Since $\mathcal{F}^{\star}_0$ is a DSW, it follows that $\mathcal{F}^{\star}_k=\mathcal{F}^{\star}_0,\ \forall k\geq 0$.

\textit{Step 3 $\left(r^0_k=\hat{\gamma}\left(\hat{\theta}_k\right), \forall k\geq 0\right)$:} 
Assume $r^0_{k}=\hat{\gamma}(\hat{\theta}_{k})$ and hence $r_{k}(s)=\hat{\gamma}(\hat{\theta}_{k}+s)$. From Step 2 we have $r^0_{k+1}=r^+_k$. When in \textbf{SBC MODE}, $r^0_{k+1}=r^0_k=\hat{\gamma}(\hat{\theta}_{k})=\hat{\gamma}(\hat{\theta}_{k+1})$. When in \textbf{MPC MODE}, $r^0_{k+1}=r_k(w^*_{k,0}\Delta t)=\hat{\gamma}(\hat{\theta}_k+w^*_{k,0}\Delta t)=\hat{\gamma}(\hat{\theta}_{k+1})$. Thus, $r^0_k=\hat{\gamma}(\hat{\theta}_k)\Rightarrow r^0_{k+1}=\hat{\gamma}(\hat{\theta}_{k+1})$.
Now, $r_0^0=r^+_0=\hat{\gamma}(0)=\hat{\gamma}(\hat{\theta}_0)$ and we can conclude $r^0_k=\hat{\gamma}(\hat{\theta}_k), \forall k\geq 0$.

\textit{Step 4 ($\exists j < \infty\ s.t.\ r^0_k=p^g, \forall k \geq j$):} 
Let $\hat{\theta}^g$ be the arc length of $\hat{\Gamma}$ such that $\hat{\gamma}(\hat{\theta})=p^g\ \forall \hat{\theta} \geq \hat{\theta}^g$. Such a $\hat{\theta}^g$ exists due to the converging properties of \eqref{eq:receding_reference_dynamics} in a DSW. 
Let $K=\left\lceil \frac{\hat{\theta}^g}{\lambda\bar{\rho}} \right\rceil$. If at time step $k$ the controller has been in \textbf{MPC MODE} at $K$ previous time iterations, it follows from \eqref{eq:theta_dynamics} and \eqref{eq:mpc_w0_con} that $\hat{\theta}_k \geq K\lambda\bar{\rho}\geq\hat{\theta}^g$. Now assume $\hat{\theta}_k < \hat{\theta}^g\ \forall k$. A solution to \eqref{eq:mpc} can then be found at most $K-1$ times, i.e. there exists a $k'<\infty$ where \eqref{eq:mpc} is infeasible for any $k\geq k'$. Then $r^0_k=r_{k'}^0, \mathcal{P}_k=\mathcal{P}_{k'}, \forall k\geq k'$ and the SBC is applied from this time instance, $u(t)=\kappa(x(t),r^0_{k'}), \forall t\geq t_{k'}$. 
Define $\epsilon=\max_{s\in[0,\lambda\bar{\rho}]}\lVert r^0_{k'}-r_{k'}(s)\rVert_2$. Due to the normalized dynamics \eqref{eq:receding_reference_path} we have $\epsilon\leq\lambda\bar{\rho}<\bar{\rho}$. Since $\lim_{t\rightarrow\infty} p(t)=r^0_{k'}$ according to the definition of the SBC, there exists a finite $k''<\infty$, s.t. $p_{k''}\in\mathbb{B}[r^0_{k'},\bar{\rho}-\epsilon]$ due to the continuity of the solution. 
Now consider the solution for \eqref{eq:mpc} $w_{0}=\frac{\lambda\bar{\rho}}{\Delta t}$, $w_{i}=0, \forall i>0$ and $\bar{u}_{i}=u', \forall i$ with $f(x,u')=0,\ \forall x$.
This is a feasible solution at time instance $k''$ since $\varepsilon(\tau)\leq \max_{s\in[0,\lambda\bar{\rho}]}\lVert p_{k''}-r_{k'}(s)\rVert_2 \leq \bar{\rho},\ \forall \tau\in[0, T]$. This is a contradiction to the conclusion that \eqref{eq:mpc} is infeasible for $k\geq k'$ and no such $k'$ exists. Thus, $\exists k<\infty \textnormal{ s.t. } \hat{\theta}_k\geq \hat{\theta}^g$. Since $w^*_{k,0}>0$, we have $\hat{\theta}_{k+1}\geq \hat{\theta}_k, \forall k$ and thus $\exists j<\infty \textnormal{ s.t. } \hat{\theta}_k\geq \hat{\theta}^g, \forall k\geq j$. Then, from Step 3 and the definition of $\hat{\theta}_g$, it follows that $\exists j<\infty \textnormal{ s.t. } r^0_k=p^g, \forall k\geq j$.

\textit{Step 5 ($\lim_{t\rightarrow\infty} p(t)=p^g$):}
From step 4 we have that there exists some time instance $j$ where $r^0_k=p^g,\ \forall k\geq j$. This implies $r^g_k=p^g$ from this point and the controller stays in \textbf{SBC MODE}. 
Since the SBC renders asymptotically stable closed-loop error dynamics, it can be concluded that $\lim_{t\rightarrow\infty} p(t)=p^g$.

\bibliographystyle{IEEEtran}
\bibliography{references}



\end{document}